\documentclass[12pt]{amsart} 

\usepackage{amssymb}
\usepackage{amsfonts}
\usepackage{amsmath}
\usepackage{bbm}

\usepackage{epsfig}
\usepackage{amstext}
\usepackage{amsthm}
\usepackage{bm}
\usepackage{eucal} 

\numberwithin{equation}{section}

\usepackage{a4wide}
\usepackage{hyperref}

\title{Optimal Rates For Regularization Of Statistical Inverse Learning Problems}
\author{Gilles Blanchard
        \and Nicole M\"{u}cke \\
        }

\address{Institute of Mathematics, University of Potsdam, Karl-Liebknecht-Straße 24-25 14476 Potsdam, Germany}
\email{\{blanchard,muecke\}@uni-potsdam.de}

\date{\today}

\theoremstyle{plain}

\newtheorem{theo}{Theorem}[section]

\newtheorem{lem}[theo]{Lemma}

\newtheorem{prop}[theo]{Proposition}

\newtheorem{cor}[theo]{Corollary}

\theoremstyle{definition}

\newtheorem{defi}[theo]{Definition}

\newtheorem{assumption}[theo]{Assumption}

\newtheorem{rem}[theo]{Remark}

\theoremstyle{remark}
\newtheorem{example}[theo]{Example}

\newcommand{\cal}{\mathcal}

\newcommand{\E}{{\mathbb{E}}}
\newcommand{\PP}{{\mathbb{P}}}
\newcommand{\R}{{\mathbb{R}}}
\newcommand{\N}{{\mathbb{N}}}

\newcommand{\lam}{\lambda}

\newcommand{\es}{{\cal S}}

\newcommand{\h}{{\cal H}_1}
\newcommand{\hh}{{\cal H}_2}
\newcommand{\hhh}{{\cal H}}
\newcommand{\X}{{\cal X}}

\newcommand{\prf}{\begin{proof}} 
\newcommand{\prfend}{\end{proof}} 
\newcommand{\Y}{{\cal Y}}
\newcommand{\la}{\langle}
\newcommand{\ra}{\rangle}

\newcommand{\x}{{\bf x}}
\newcommand{\y}{{\bf y}}
\newcommand{\z}{{\bf z}}
\newcommand{\Z}{{\bf Z}}
\newcommand{\M}{{\cal M}}

\newcommand{\K}{{\cal K}}
\newcommand{\NN}{{\cal N}}
\newcommand{\PPP}{{\cal P}}
\newcommand{\eps}{\varepsilon}

\newcommand{\pfeil}{\longrightarrow}

\newcommand{\eigv}{\mu}

\newcommand{\rad}{\pi}

\newcommand{\Ran}{\mathrm{Im}}

\newcommand{\nux}{\nu}
\newcommand{\nuemp}{\widehat{\nu}}

\newcommand{\fo}{f_{\rho}}

\newcommand{\fest}{\widehat{f}_n}

\newcommand{\priorle}{\PPP^<}
\newcommand{\priorgr}{\PPP^>}

\newcommand{\hs}{{\mathrm{HS}}}
\newcommand{\tr}{\mathrm{tr}}

\newcommand{\vspan}{\mathrm{Span}}

\newcommand{\wt}{\widetilde}
\newcommand{\wh}{\widehat}
\newcommand{\ol}{\overline}

\newcommand{\ind}[1]{{\mathbf{1}\{#1\}}}
\newcommand{\paren}[1]{\left(#1\right)}
\newcommand{\brac}[1]{\left[#1\right]}
\newcommand{\inner}[1]{\left\langle#1\right\rangle}
\newcommand{\norm}[1]{\left\|#1\right\|}
\newcommand{\snorm}[1]{\left\| B^s(#1)\right\|}
\newcommand{\set}[1]{\left\{#1\right\}}
\newcommand{\abs}[1]{\left\lvert #1 \right\rvert}

\RequirePackage{color}

\newcounter{nbdrafts}
\makeatletter
\newcommand{\checknbdrafts}{
\ifnum \thenbdrafts > 0
\@latex@warning@no@line{**********************************************************************}
\@latex@warning@no@line{* The document contains \thenbdrafts \space draft note(s)}
\@latex@warning@no@line{**********************************************************************}
\fi}

\makeatother

\begin{document}


\begin{abstract}
We consider a statistical inverse learning problem, where we observe the image of a function $f$
through a linear operator $A$ at i.i.d. random design points $X_i$, superposed with an additive noise.
The distribution of the design points is unknown and can be very general.
We analyze simultaneously the direct (estimation of $Af$) and the inverse (estimation of $f$) learning problems. 
In this general framework, we obtain strong and weak minimax optimal rates of convergence (as the number of observations $n$ grows large) 
for a large class of spectral regularization methods over regularity classes defined through appropriate source conditions.
This  improves on or completes previous results obtained in related settings. 
The optimality of the obtained  rates is shown not only in the exponent in $n$ but also in the explicit
dependency of the constant factor in the variance of the noise and the radius of the source condition set.
\end{abstract}

\maketitle


\section{Introduction}

\subsection{Setting}
Let $A$ be a known linear operator from a Hilbert space  $\hhh_1$ to
a linear space $\hhh_2$
of real-valued functions over some input space $\X$. In this paper we
consider a random and noisy observation scheme of the form
\begin{equation}
  \label{eq:learnmodel}
  Y_i:=g(X_i) + \eps_i \; , \quad g=Af \; , \quad i=1\,,\ldots,n
  \end{equation}
at i.i.d.
data points $X_1,\ldots,X_n$\, drawn according to a probability distribution
$\nux$ on $\X$, where $\eps_i$ are independent centered noise variables.
More precisely, we assume that the observed data $(X_i,Y_i)_{1 \leq i \leq n}$ are i.i.d. observations,
with $\E[Y_i|X_i]=g(X_i)$, so that the distribution of $\eps_i$ may
depend on $X_i$\,, while satisfying $\E[\eps_i|X_i]=0$\,. This is also commonly called a \emph{statistical learning} setting, in the
sense that the data $(X_i,Y_i)$ are generated by some external random source and the  \emph{learner} aims to infer from the data some
reconstruction $\fest$  of $f$, without having influence on the underlying sampling distribution $\nux$.
For this reason we call model \eqref{eq:learnmodel} an {\em inverse statistical learning problem}.
The special case $A=I$ is just non-parametric regression under random design (which we also call the direct problem).
Thus, introducing a general $A$ gives a 
unified approach to the direct and inverse problem. 

In the statistical learning context, the relevant notion of convergence and associated reconstruction rates to recover $f$ concern the limit $n \rightarrow \infty$\,.
More specifically, let $\fest$ be an estimator of $f$ based on the observed data  $(X_i,Y_i)_{1 \leq i \leq n}$. 
The usual notion of estimation error in the statistical learning context is the averaged
squared loss for the prediction of $g(X)$ at a new independent sample point $X$:
\begin{equation}
\label{error}
\E_{X\sim\nux}[(g(X)-A\fest(X))^2] = \norm{A(f- \fest)}^2_{L^2(\nux)}\,.
\end{equation}
In this paper, we are interested as well in the {\em inverse} reconstruction problem,
that is,  the reconstruction error for $f$ itself in the input space norm, i.e.
\[
\norm{f- \fest}^2_{\hhh_1}\,.
\]
Estimates in $L^2(\nux)$-norm are standard in the learning context, while estimates in $\hhh_1$-norm are standard
for inverse problems, and our results will present convergence results for a family of norms interpolating
between these two.
We emphasize that $\norm{A(f- \fest)}^2_{L^2(\nux)}$ as well as $\norm{f- \fest}^2_{\hhh_1}\,$
are  random variables, depending on the observations. 
Thus the error rates above can be estimated either in expectation or in probability.
In this paper we will present convergence rates for these different criteria,
as $n$ tends to infinity, both in expectation (for moments of all orders) and
with high probability.

\subsection{Overview of the results}

\label{se:overview}

In this section we present a short, informal overview of the results which will allow a comparison to other existing results in the next section.
We start to show that, under appropriate assumptions, we can endow $\Ran(A)$ with an appropriate
reproducing kernel Hilbert space (RKHS) structure $\hhh_K$ with reproducing kernel $K$,
such that $A$ is a partial isometry from $\h$ onto
$\hhh_K$\,. Through this partial isometry the initial problem \eqref{eq:learnmodel} can be formally reduced
to the problem of estimating the function $g\in \hhh_K$ by some $\wh{g}$\,; control of the error
$(g-\wh{g})$ in $L^2(\nux)$-norm corresponds to the direct (prediction) problem, while control
of this difference in $\hhh_K$-norm is equivalent to the inverse (reconstruction) problem. In particular, the kernel $K$ completely encapsulates the information
about the operator $A$\,.
This equivalence
also allows a direct comparison to previous existing results for convergence rates of statistical learning using a RKHS formalism (see next section). Let $L: g \in \hhh_K \mapsto \int g(x) K(x,.) d\nux(x) \in \hhh_K$
denote the kernel integral operator associated to $K$ and the sampling measure $\nux$. The rates of convergence
presented in this paper will be governed by a {\em source condition} assumption on $g$ of the form
$\norm{L^r g} \leq R$ for some constants $r,R>0$ as well as by the {\em ill-posedness} of the problem,
as measured by an assumed power decay of the eigenvalues of $L$ with exponent $b>1$\,. Our main
upper bound result  establishes that for a broad class of estimators defined via spectral
regularization methods, for $s\in[0,\frac{1}{2}]$ it holds both with high probability as well as in the sense of $p$-th moment expectation that
\[
\norm{L^s(g - \wh{g}_{\lam_n})}_{\hhh_K}  \lesssim R \paren{\frac{\sigma^2}{R^2n}}^{\frac{(r+s)}{2r+1+1/b}}\,,
\]
for an appropriate choice of the regularization parameter $\lam_n$\,.  
(Note that $s=0$ corresponds to reconstruction error, and
$s=\frac{1}{2}$ to the prediction error i.e. $L^2(\nux)$ norm)\,.
Here $\sigma^2$ denotes noise variance (classical Bernstein moment conditions
are assumed to hold for the noise.)
The symbol $\lesssim$ means that the inequality holds up to a multiplicative constant that can depend on various
parameters entering in the assumptions of the result, but not on $n$, $\sigma$, nor $R$\,. 
An important assumption is that the inequality $q \geq r+s$ should hold, where $q$ is the
{\em qualification} of the regularization method, a quantity defined in the classical theory
of inverse problems (see Section~\ref{se:regulariz} for a precise definition)\,.

This result is complemented
by a minimax lower bound which matches the above rate not only in the exponent in $n$\,,
but also in the precise behavior of the multiplicative constant in function of $R$ and the noise
variance $\sigma^2$\,. The obtained lower bounds come in two flavors, which we call weak and a strong
asymptotic lower bound (see Section \ref{sec:mainresults}).

\subsection{Related work}
The analysis of inverse problems, discretized via (noisy) observations at a finite
number of points, has a long history, which we will not attempt to cover in detail here.
The introduction of reproducing kernel Hilbert space based methods was a crucial step
forward in the end of the 1970s. Early references have focused, mostly,
on spline methods on $[0,1]^d$\,; on observation point designs either
deterministic regular, or random with a sampling probability comparable to Lebesgue;
and on assumed regularity of the target function in terms of usual
differentiability properties. We refer to \cite{Wahba} and references therein for
a general overview.
An early reference establishing convergence rates in a random design setting for
(possibly nonlinear) inverse problems in a setup similar to those delineated above
and a Tykhonov-type regularization method is \cite{osullivan}.
Analysis of the convergence of fairly general regularization schemes for statistical
inverse problems under the {\em white noise} model were established in
\cite{bissantz}.  The white noise model is markedly different from the setting
considered in the present paper, in particular because it does not involve randomly sampled observation points, though as a general rule one expects a correspondence
between optimal convergence rates in both settings.

We henceforth focus our attention on the more recent thread of literature concerning
the {\em statistical learning} setting, whose results are more directly comparable to ours.
In this setting, the emphasis is on general input spaces, and ``distribution-free'' results,
which is to say, random sampling whose distribution $\nux$ is unknown, quite arbitrary
and out of the control of the user. The use of reproducing kernel methods have enjoyed a wide
popularity in this context since the 1990s, mainly for the direct learning problem.
The connections between (the direct problem of) statistical learning using
reproducing kernel methods, and inverse problem methodology, were first noted and studied in \cite{learning,rosasco,discretization}. In particular, in \cite{rosasco} it was
proposed to use general form regularization methods from the inverse problem literature for kernel-based
statistical learning. There is a vast recent literature relating learning to regularization techniques for inverse problems (see \cite{mendelson}, \cite{zhouwang}, \cite{poggio} to mention just a few), confirming the strong conceptual analogy of certain learning algorithms with regularization algorithms. For example, 
Tikhonov regularization is known as regularized least-squares algorithm or ridge regression, while  Landweber iteration is related to $L^2$-boosting or gradient descent, see e.g. \cite{Yao} and \cite{buhlmann}.    

In \cite{discretization}, the more general setting 
of the random discretization of an inverse problem defined by a Carleman operator is considered. This is essentially the setting we
adopt in the present work. More precisely, we start with the assumption that the map
$(f,x) \mapsto (Af)(x)$ is continuous in $f$ and measurable in $x$\,, which implies that
$A$ can be seen as a Carleman operator from $\hhh_1$ to $L^2(\nux)$\,.
Moreover, as mentioned in the previous section, we observe that
$\Ran(A)$ can be endowed with a RKHS structure $\hhh_K$ 
such that $A$ is a partial isometry from $\h$ onto $\hhh_K$\,. While we do not expect this result to considered
a novelty, it was not explicitly mentioned in \cite{discretization} and in our opinion helps 
cement the equivalence between inverse statistical learning and direct learning with reproducing kernels.
In particular, it makes a direct comparison possible between our results and previous results for the
direct (kernel) learning problem.


Concerning the history of upper rates of convergence
in a RKHS setting, covering number techniques were used in \cite{cusmale} to obtain 
(non-asymptotic) upper rates. In \cite{learning}, \cite{smalezhou1}, \cite{smalezhou2} these techniques were replaced by estimates on integral operators via concentration inequalities, and this is the path we follow in this paper.

We shall now briefly review previous results which are directly comparable to ours: Smale and Zhou \cite{smalezhou2}, Bauer et al. \cite{per}, Yao et al. \cite{Yao}, Caponnetto and De Vito \cite{optimalratesRLS} and Caponnetto \cite{optimalrates}. For convenience, we have tried to condense the most 
essential points in Table \ref{ta:table1}. Compared with our more general setting, all of these previous references only consider the special case 
$A=I$, but assume from the onset that $\hhh_1$ is a RKHS with given kernel. Thus, in the first column of Table~\ref{ta:table1}, $A$ 
is the identity and $g=f$\,, and in the second column
$\hhh_1 = \hhh_K$. The more complicated form given in Table~\ref{ta:table1} is the
reinterpretation in our setting (see Section 2). The first three references (\cite{smalezhou2}, \cite{per}, \cite{Yao}) do not analyze lower bounds and their upper bounds do not take into account the behaviour of the eigenvalues of the 
integral operator $L$ corresponding to the assumed RKHS structure. But all three derive estimates on the error both in 
$L^2(\nux)$-norm and RKHS-norm. Only \cite{per} considers a general class of spectral regularization methods.

The last two papers \cite{optimalratesRLS} and 
\cite{optimalrates} obtain fast upper rates (depending on the eigenvalues of $L$) which are minimax optimal. The estimates, however, are only given in 
$L^2(\nux)$-norm. Furthermore, only \cite{optimalrates} goes beyond Tikhonov
regularization to handle a general class of spectral regularization methods.
A closer look at Table~\ref{ta:table1} reveals that in treating general spectral regularization methods, the results of \cite{optimalrates} require
for certain parameter configurations ($r<1/2-1/2b$)
the availablility of additional unlabeled data from the sampling distribution $\nux$\,.
This appears somewhat suboptimal, since this does not reproduce the previously obtained result for Tikhonov in \cite{optimalratesRLS} which does not
require unlabeled data.

To obtain these finer results (``fast rates'' taking into account the spectral
structure of $L$), a crucial technical tool is to consider the \emph{effective dimension} $\cal{N}(\lam)=\tr((L+\lam)^{-1}L)$\,,
which determines the optimal choice of the 
regularization parameter. This idea of \cite{optimalratesRLS} and \cite{optimalrates} is fundamental for our approach, which extends and refines these previous results. 

Furthermore, we recall from \cite{optimalratesRLS} that the effective dimension $\cal{N}(\lam)$ seems to be just the right parameter to establish an 
important connection between the operator theoretic and spectral methods and the results obtained via entropy methods (see \cite{devore}, \cite{temlyakov}) 
since $\cal{N}(\lam)$ encodes via $L$ crucial properties of the marginal distribution $\nu$.




\begin{table}
{\small 
\begin{tabular}{|r|c|c|c|c|}
\hline
\vspace{-0.2cm}
&$\norm{  A(\fest- f)}_{L^2(\nux)}$ & $\norm{ \fest - f }_{\hhh_K}$ & Assumptions  & Method \\
&&&($q$: qualification)&\\
\hline
Smale/ Zhou \cite{smalezhou2}&$\left(\frac{1}{\sqrt{n}}\right)^{\frac{2r+1}{2r+2}}$&$\left(\frac{1}{\sqrt{n}}\right)^{\frac{r}{r+1}}$& $r \leq \frac{1}{2}$& Tikhonov\\
\hline
Bauer  et al. \cite{per}&$\left(\frac{1}{\sqrt{n}}\right)^{\frac{2r+1}{2r+2}}$&$\left(\frac{1}{\sqrt{n}}\right)^{\frac{r}{r+1}}$&$r \leq q-\frac{1}{2}$& General\\
\hline
Yao et al. \cite{Yao}&$\left(\frac{1}{\sqrt{n}}\right)^{\frac{2r+1}{2r+3}}$&$\left(\frac{1}{\sqrt{n}}\right)^{\frac{r}{r+\frac{5}{2}}}$& $q=\infty$& Landweber \\
&&& & Iteration\\
\hline
Caponnetto, De Vito \cite{optimalratesRLS}&$\left(\frac{1}{\sqrt{n}}\right)^{\frac{(2r+1)}{2r+1+\frac{1}{b}}}$& N/A &$r\leq \frac{1}{2}$&Tikhonov\\
\hline
Caponnetto \cite{optimalrates}&$\left(\frac{1}{\sqrt{n}}\right)^{\frac{(2r+1)}{2r+1+\frac{1}{b}}}$& N/A &$
r \leq q - \frac{1}{2}$&General\\
&&&+unlabeled data&\\
&&& if $2r + \frac{1}{b} <1 $&\\
\hline
\end{tabular}
\caption{\label{ta:table1} Upper rates available from earlier literature (for their applicability
  to the inverse learning setting considered in the present paper, see Section~\ref{se:equivalence}).
}}
\end{table}





As delineated in Section \ref{se:overview}, the main question adressed in this paper is that of minimax optimal rates of convergence
as $n$ grows to infinity. Our contribution is to improve on and extend the existing results presented above,  aiming to present a complete picture.
We consider a unified approach which allows to simultaneously treat the direct and the inverse learning problem, derive upper bounds 
(non-asymptotic and asymptotic) as well as lower bounds, both for the $L^2(\nux)$ and the $\hhh_1$ norm (as well as intermediate norms)
for a general class of regularization methods, without requiring additional unlabeled data.
In this generality, this is new. 
In addition, we present a refined analysis of (both strong and weak) minimax optimal rates also investigating their dependence on the complexity of the source condition and on
the variance of the noise (our lower bounds come in slightly different strong and weak versions leading to the natural notion of weak and strong minimax optimality). To the best of our knowledge, this has never been done before.

We conclude this review by mentioning the recent work \cite{loustau2013}, which also concerns inverse statistical learning
(see also \cite{loustau2015}), albeit in a quite different setting. In that work, the main focus is on classification ($Y$
only can take finitely many values or ``classes''), and the inverse problem is that the sampling distribution for $X$ is transformed via
a linear operator $A$. The method analyzed there is empirical risk minimization using a modified loss which implicitly  
includes an estimation of the original class-conditional distributions from the transformed ones. In the present paper, we consider
an (inverse) regression setting with a continuous output variable, the nature of the inverse problem is different since the
transformation is applied to the regression function, and we also use a different methododological approach.

The outline of the rest of the paper is as follows. In Section~\ref{se:notation}, we fix notation and describe our setting in more detail. In particular, we adopt the theory of Carleman operators from the direct problem to our more general setting, including the inverse learning problem. We describe the source conditions, the assumptions on the noise and prior classes, and finally the general class of spectral regularization methods. Granted these preliminaries, we then present in 
Section~\ref{sec:mainresults} our main results (Theorem~\ref{maintheo3}, Theorem~\ref{minimaxlowerrate} and Corollary~\ref{maincor2}). 
In Section~\ref{se:disc}, we present a concluding discussion on some further aspects of the
results.
Section~\ref{se:proofupper} contains the proofs of the upper bounds,
Section~\ref{se:prooflower} is devoted to the proof of 
lower bounds. 
In the Appendix we establish the concentration inequalities and a perturbation result needed in Section~\ref{se:proofupper} and give some supplementary technical lemmata 
needed in section Section~\ref{se:prooflower}.


\section{Notation and Preliminaries}

\label{se:notation}

In this section, we specify the mathematical setting and assumptions for the model \eqref{eq:learnmodel} and reduce it to an equivalent model.

\subsection{Inverse Problems induced by Carleman Operators}
\label{se:setting}


We assume that the input space $\X$ is a standard Borel space
endowed with a probability measure $\nux$,
and the output space $\Y$ is equal to $\R$.
Let $A: \h \longrightarrow \hh$ be a linear operator, were $\h$ is a
infinite-dimensional 
real separable Hilbert
space 
and $\hh$ some 
vector space of functions $g: \X \longrightarrow \R$.
We don't assume any specific structure on $\hh$ for now.
However, as will become clear shortly, the image $\Ran(A) \subset \hh$ 
will be endowed with a natural Hilbert space structure as a consequence
of following key assumption:
\begin{assumption}
\label{assumpteval}
The evaluation functionals at a given point $x\in \X$\,:
\begin{eqnarray*}
S_x: \h &\longrightarrow & \R \\
       f  & \longmapsto & (S_x)(f):=(Af)(x) 
\end{eqnarray*}
are uniformly (w.r.t. $x \in \X$) bounded, i.e., there exists a constant
$\kappa< \infty$ such that for any $x \in \X$
\[ | S_x(f) | \; \leq \; \kappa\; \norm{f}_{\h} \; .  \]
\end{assumption}
For all $x$, the fact that $S_x$ is continuous 
implies, by Riesz's representation theorem,
the existence of an element $F_{x} \in \h$ such that
\[ (Af)(x) = \inner{f, F_{x}}_{\h}  \]
with
\[ \norm{F_{x}}_{\h} = \norm{S_x} \leq \kappa\; , \] 
for any $x \in \X$. Define the map
\begin{eqnarray*}
K: \X \times \X &\longrightarrow & \R \\
   (x_1,x_2)&\longmapsto & K(x_1,x_2):=\inner{ F_{x_1}, F_{x_2}}_{\h} \,,
\end{eqnarray*}
which is by construction a positive semidefinite (p.s.d.) kernel over $\X$ associated to
the so-called feature space $\h$, and the feature map $F_{\cdot}: x \in \X \mapsto F_{x} \in \h$. Observe that for any $x\in \X$, we have the bound
$K(x,x) = \norm{F_x}_{\h}^2 \leq \kappa^2$\,.
A fundamental result (see \cite{steinwart}, Theorem $4.21$) is that to every 
p.s.d. kernel can be associated a unique reproducing kernel Hilbert space (RKHS).
We reproduce this result here, adapted to the considered context:
\begin{prop}(Unique RKHS associated to a psd kernel)
  \label{prop:rkhsprop}
The real-valued function space
\begin{align*}
\hhh_K & := \{ g:\X \longrightarrow \R \; | \; \exists \; f \in \h \;{\rm with} \;\; g(x)=\inner{ f, F_{x}}_{\h} = (Af)(x) \; \forall x \in \X  \}  \\
& = \Ran(A) \subset \hh,
\end{align*}
equipped with the norm
\begin{align*}
 \norm{g}_{\hhh_K} & := \inf\set{ \norm{f}_{\h} \;:\; f\in \h \; {\rm s.t.}\; 
\forall x \in \X\,: g(x)=\inner{ f, F_x }_{\h} = (Af)(x)} \\
& = \inf_{f \in A^{-1}(\set{g})} { \norm{f}_{\h}}
\end{align*}
is the unique RKHS for which $K$ is a reproducing kernel.
Moreover, the operator
$A$ 
is a partial isometry from $\h$ to $\hhh_K$ (i.e. an isometry on the orthogonal of its kernel),
and 
\[ \hhh_K = \overline{{\vspan}}\{ K(x,.), x \in \X  \} \; .  \]
\end{prop}
From now on, we can therefore forget about the space $\hh$ and
consider $A$ as an operator from $\h$ onto $\hhh_K = \Ran(A)$.
As a consequence of $A$ being a partial
isometry onto $\hhh_K$, note that this RKHS is separable, since we have assumed that $\h$ is.
Additionally, we assume 
\begin{assumption}
\label{assumptCarleman}
For any $f \in \h$, the map $x \mapsto (Af)(x) = \inner{f ,F_{x}}_{\h}$ is measurable.
\end{assumption}
Equivalently, it is assumed that all functions $g \in \hhh_K$ are
measurable. Furthermore, Assumption \ref{assumpteval}
implies that $\norm{Af}_\infty \leq \kappa \norm{f}_{\h}$ for all
$f \in \h$, so that all functions in $\hhh_K$ are bounded in
supremum norm. Therefore,
$\hhh_K$ is a subset of $L^2(\X,\nux)$\,; let $\iota$ denote the
associated canonical injection map $\hhh_K \hookrightarrow L^2(\X,\nux)$\,.

Together, Assumptions \ref{assumptCarleman} and \ref{assumpteval} thus imply that the map $F_{\cdot}: \X \longrightarrow \h$ is a bounded {\it Carleman map} 
\cite{halmos}. We define the associated {\it Carleman operator},
as
\begin{align*}
S_\nux: \h & \longrightarrow L^2(\X, \nux) \\
f & \longmapsto  S_\nux f := \iota(Af)\,.
\end{align*}
The operator $S_\nux$ is bounded and satisfies $\norm{S_\nux}\leq \kappa$\,,
since
\[ \norm{S_\nux f}^2_{L^2(\nux)}  =  \int_{\X} |(Af)(x)|^2 \; \nux(dx) 
=  \int_{\X} |\la f,F_{x} \ra_{\h} \;|^2  \; \nux(dx) 
\leq \kappa^2  \norm{f}^2_{\h}  \; .
\]

We give an illustrative example which is very classical. 

\begin{example}
\label{differentiation} 
(Differentiating a real function)
We consider estimation of a derivative of a real function. To this end, 
we let $\h:=\{ f \in L^2[0,1]: \; \E[f]=0 \}$, the subspace of $L^2([0,1], dt)$ consisting of functions with mean zero 
and $\hh:=C[0,1]$, the space of continuous functions on $[0,1]$. 
Define $A: \h \longrightarrow \hh$ by
\[    [Af](x) = \int_0^x f(t)\; dt  \;. \]
Then $Af = g$ if and only if $f = g'$. It is easily checked that Assumption \ref{assumpteval} is satisfied. To identify the kernel of $\Ran(A)$, 
the reader can easily convince himself that
\[  [Af](x) = \la f, F_x \ra_{L^2} \;,\]
where $F_x(t)=\mathbbm{1}_{[0,x]}(t) - x$. Thus, by definition $K(x,t)= \la F_x, F_t \ra_{L^2} = x\wedge t-xt$ and $\Ran(A)$ coincides with the real 
Sobolev space $H_0^1[0,1]$, consisting of absolutely continuous functions $g$ on $[0,1]$ with weak derivatives of order $1$ in $L^2[0,1]$, 
with boundary condition  $g(0)=g(1)=0$. The associated Carleman operator is given by $S= \iota \circ A$   
with $\iota : H^1_0[0,1] \hookrightarrow L^2[0,1]$ and with marginal distribution $\nu=dt$, the Lebesgue measure on $[0,1]$.  
\end{example}

We complete this section by defining $B_\nux:=S_\nux^{\star}S_\nux: \h\longrightarrow \h$. Then $B_\nux$ is positive, selfadjoint and satisfies $\norm{B_\nux} \leq \kappa^2$\,.
The following Proposition  summarizes the main properties of the operators $S_\nux,S^{\star}_\nux$ and $B_\nux$. Its proof can be found in the Appendix of \cite{discretization} (Proposition 19).

\begin{prop}
\label{Carlemanoperator}
Under Assumptions \ref{assumpteval} and \ref{assumptCarleman}, the 
Carleman operator $S_\nux: \h \longrightarrow L^2(\X, \nux)$ is a 
Hilbert-Schmidt operator with nullspace
\[ \ker(S_\nux) = \vspan\{F_{x} \; : \; x \in \mathrm{support}(\nux) \}^{\bot} \; .  \]
The adjoint operator $S_\nux^{\star}: L^2(\X, \nu)\longrightarrow \h $ is given by
\[ S_\nux^{\star}g = \int_{\X} g(x)F_{x}  \; \nu(dx) \; ,\]
for any $g \in L^2(\X, \nu)$ and where the integral converges in $\h$- norm.\\
Furthermore, if $F_x \otimes F_x^\star$ denotes the operator
$f \in \h \mapsto \inner{f,F_x}_{\h}F_x \in \h$, then
\[ B_\nux = \int_{\X} F_x \otimes F_x^\star  \; \nu(dx) \; , \]
where the integral converges in trace norm. 
\end{prop}

It is natural to consider the inverse problem $S_\nux f=g$
(rather than $Af=g$) as the idealized population version (i.e. noise
and discretization-free) of \eqref{eq:learnmodel}, since
since the former 
views the output of the operator in the geometry of $L^2(\X,\nux)$,
which is the natural population geometry when the sampling measure is
$\nux$\,.
Multiplying on both sides by $S_\nux^\star$\,, we obtain the inverse problem
$B_\nux f = S_\nux^\star g$ (called ``normal equation'' in the inverse problem literature).

Since $B_\nux$ is self-adjoint and compact,  
the spectral theorem ensures the existence of an orthonormal set $\{e_j\}_{j\geq 1}$ such that 
\begin{equation}
\label{eq:eigdec}
B_\nux = \sum_{j =1}^\infty  \eigv_j \la \cdot , e_j\ra_{\h}e_j
\end{equation}
and
\[ \h = \ker(B_\nux) \oplus \overline{\vspan}\{e_j\; : \; j \geq 1 \} \; .\]
The numbers $\eigv_j$ are the positive eigenvalues of $B_\nux$ in decreasing order, satisfying 
$0< \eigv_{j+1} \leq \eigv_j$ for all $j>0$\, and $\eigv_j \searrow 0$.
In the special case where $B_\nux$ has finite rank,
the above set of positive eigenvalue and eigenvectors is finite, but to simplify the notation
we will always assume  that they are countably infinite; formally, we can accomodate for this
special situation by allowing that the decreasing sequence of eigenvalues is equal to zero
from a certain index on.

\begin{rem}
The considered operators depend on the sampling measure $\nux$ and thus also 
the eigenvalues $(\eigv_j)_{j\geq 1}$\,. For the sake of reading ease, we omit this dependence in the notation; we will also
denote henceforth $S=S_{\nux}$ and $B=B_{\nux}$\,.
\end{rem}


\subsection{Discretization by random sampling}

For discretization, we consider a sample $\z=(\x, \y)=((x_1,y_1),...,(x_n,y_n)) \in (\X \times \R)^n$ and introduce the associated {\it sampling operator}
\begin{eqnarray*}
 S_{\x}: \h & \longrightarrow & \R^n \\
           f & \longmapsto & S_{\x}f\,,
\end{eqnarray*} 
with $(S_{\x}f)_j = \la f, F_{x_j} \ra_{\h}$, $j=1,...,n$ and where $\R^n$ is equipped with the inner product
of the empirical $L^2$ structure,
\[ \la \y, \y' \ra_{R^n} = \frac{1}{n}\sum_{j=1}^ny_jy_j'\,. \]
Formally, $S_{\x}$ is the  counterpart of $S_\nux$ when replacing the sampling
distribution $\nux$ by the empirical distribution $\nuemp := \frac{1}{n} \sum_{i=1}^n \delta_{x_i}$\,,
and identifying $L^2(\X,\nuemp)$ with $\R^n$ endowed with the above inner product.
Additionally, the sampled vector $S_{\x}f$ is corrupted by noise 
$\bm{\eps} = (\epsilon_1,\ldots,\epsilon_n)$ to yield
the vector of observed values $\y = (y_1,...,y_n) \in \R^n$:
\begin{equation}
\label{samplecase}
y_j = g(x_j) + \eps_j = (S_{\x}f)_j + \eps_j ,\;\; j=1,...,n\,,
\end{equation}
which can be interpretet as the discretized and noisy counterpart of the population problem $S_\nu f=g$\,.
Replacing the measure $\nux$ with the empirical measure $\nuemp$
in Proposition~\ref{Carlemanoperator}
gives the following Corollary:

\begin{cor}
\label{samplingop}
The sampling operator $S_{\x}: \h  \longrightarrow  \R^n $ is a Hilbert-Schmidt operator with nullspace 
\[ \ker(S_{\x})=\vspan \{F_{x_j}\; :\; j=1,...,n\}^{\bot} \; .\]
Furthermore, the adjoint operator $S_{\x}^{\star}: \R^n\longrightarrow \h$ is given by 
\[ S_{\x}^{\star}\y = \frac{1}{n}\sum_{j=1}^n y_j F_{x_j}\,,  \]
and the operator $B_{\x}:= S^\star_{\x}S_{\x}: \h \longrightarrow \h$ is given by
\[  B_{\x} =\frac{1}{n}\sum_{j=1}^n  F_{x_j} \otimes  F_{x_j}^{\star}  \; .\]
\end{cor}

With this notation, the normal equation associated to \eqref{samplecase},
obtained by multiplying both sides by $S_\x^\star$\,, reads
$S_\x^\star \y = B_\x f + S_\x^\star \bm{\epsilon}$\,; it is the discretized and
noisy counterpart of the population normal equation introduced in the previous section.
The advantage of looking at the normal equations is that both the population and the empirical
version act on the same space $\h$\,, so that the latter can be considered as a perturbation
of the former (both for the operator and the noise term), an observation which is central to the
theory \cite{learning}.


\subsection{Statistical model, noise assumption, and prior classes}
\label{priors}

We recall the considered setting of inverse learning, the sampling is assumed to be
random i.i.d., where each observation point $(X_i,Y_i)$ follows the model
$ Y = Af(X) + \eps \,.$
More precisely, $(X_i,Y_i)$ are i.i.d. with Borel probability distribution $\rho$ on $\X \times \R$\,. For $(X,Y)$ having distribution $\rho$,
denoting $\nux$
the marginal distribution of $X$\,, we assume:
\begin{assumption}
  \label{basicmodel}
  The conditional expectation wrt. $\rho$ of $Y$ given $X$ exists and it holds
  for $\nux$-almost all $x \in X$\,:
\begin{equation}
  \label{basicmodeleq}
  \E_\rho[Y | X=x] = A\fo(x) = S_x \fo\,, \text{ for some } \fo \in \hhh_1\,.
  \end{equation}
\end{assumption}
Furthermore, we will make the following Bernstein-type assumption on the
observation noise distribution:
\begin{assumption}
  There exists $\sigma > 0$ and $M>0$ such that for any integer $m \geq 2$:
\begin{equation}
\label{bernstein}
\E[\; \abs{Y - A\fo(X)}^{m} \; | \; X \;] \leq \frac{1}{2}m! \; \sigma^2 M^{m-2} \quad \nux - {\rm a.s.} \;
\end{equation}
\end{assumption}

It is a generally established fact that given any estimator $\wh{f}$ of $\fo$, 
one can construct a probability measure $\rho$
on $\X \times \R$ such that the rate of 
convergence of $\wh{f}$ to $\fo$ can arbitrarily slow (see e.g. \cite{gyorfi:book}). Thus, to  derive nontrivial rates of convergence, we concentrate
our attention on specific subsets (also called {\em models})  of the class of probability measures. We will work with the same type of assumptions
as  considered by \cite{optimalratesRLS} and introduce two sets of conditions
concerning, on the one hand, the marginal distribution $\nux$ of $X$, and
on the other hand, the conditional distribution $\rho(.|.)$ of $Y$ given $X$.

Let $\PPP$ denote the set of all probability distributions on $\X$.
We define classes of sampling distributions by introducing decay conditions on
the eigenvalues $\eigv_i$ of the operator $B_\nux$ defined in Section~\ref{se:setting}.

For $b>1$ and $\alpha, \beta >0$\,, we define
\[ \priorle(b, \beta):=\{ \nu \in  {\PPP}:  \; \eigv_j \leq \beta/j^b \; \; \forall j \geq 1 \}  \; , \]
\[ \priorgr(b, \alpha):=\{ \nu \in  {\PPP} :  \; \eigv_j \geq \alpha/j^b \; \; \forall j \geq 1\}  \]
and
\[  \priorgr_{strong}(b, \alpha):=\{ \nu \in \priorgr(b, \alpha): \; \exists \gamma>0\,, j_0 \geq 1 \;{\rm s.th.}\; 
      \frac{\eigv_{2j}}{\eigv_j}\geq 2^{-\gamma} \; \; \forall  j \geq j_0 \} \; .\]
In the inverse problem literature, such eigenvalue decay assumptions are related to the so-called degree of ill-posedness of the inverse problem
$B_\nux f=S^\star g$\,.
In the present setting, the ill-posedness of the problem is reflected by the eigenvalues of $B_\nux$ and depends both of
the fixed operator $A$ and the sampling distribution $\nux$.

\begin{example}
Coming back to our example \ref{differentiation} the degree of ill-posedness is determined by the decay of the eigenvalues $(\mu_j)_j$ 
of the positive selfadjoint integral operator $L_K = SS^{\star}: L^2[0,1] \longrightarrow L^2[0,1]$
\[ [L_K f](x) = \int_0^1 K(x,t)f(t)  \; dt \; . \]
Elementary calculations show that the SVD basis is given by $e_j(x)=\sqrt{2}\sin(\pi jx)$ with corresponding singular values 
$\mu_j=\frac{1}{\pi ^2j^2}$. Thus, $b=2$ and $\priorle(2, \frac{1}{\pi ^2}) \cap {\priorgr}(2, \frac{1}{\pi ^2})$ as well as 
$\priorle(2, \frac{1}{\pi ^2}) \cap {\priorgr_{strong}}(2, \frac{1}{\pi ^2})$ are not empty.
\end{example}

For a subset $\Omega \subseteq \h$, we let ${\cal K}(\Omega)$ be the set of regular conditional probability distributions
$\rho(\cdot|\cdot)$ on ${\cal B}(\R)\times \X$ such that
$(\ref{basicmodeleq})$ and $(\ref{bernstein})$ hold for some $\fo \in \Omega$.
(It is clear that these conditions only depend on the
conditional $\rho(.|.)$ of $Y$ given $X$.)
We will focus on a {\it H\"older-type source condition}, which is a classical smoothness assumption in the theory of inverse problems. Given $r>0, R>0$ and $\nu \in {\cal P}$, we define
\begin{equation}
\label{sourceset}
\Omega_{\nux}(r,R) := \{f \in \h : f = B_{\nux}^rh,\; \norm{h}_{\h}\leq 
R   \}.
\end{equation}
Note that for any $r \leq r_0$ we have $\Omega_{\nux}(r_0, R) \subseteq \Omega_{\nux}(r, \kappa^{2(r_0-r)}R)$, for any $\nux \in {\cal P}$. Since $B_{\nux}$ is compact, the source sets $\Omega_{\nux}(r,R)$ are precompact sets in $\h$. 

Then the class of models which we will consider will be defined as
\begin{equation}
\label{measureclass}
 \M(r,R,\PPP') \; := \; \{ \; \rho(dx,dy)=\rho(dy|x)\nux(dx)\; : \; 
\rho(\cdot|\cdot)\in {\cal K}(\Omega_{\nux}(r,R)), \; \nux \in  \PPP' \;\} \;,
\end{equation}
with $\PPP'=\priorle(b,\beta)$, $\PPP'=\priorgr(b,\alpha)$ or $\PPP'=\priorgr_{strong}(b,\alpha)$\,.

As a consequence, the class of models depends not only on the smoothness properties of the solution (reflected in the parameters $R>0, \; r>0$), 
but also essentially on the decay of the eigenvalues of $B_\nux$.


\subsection{Equivalence with classical kernel learning setting}
\label{se:equivalence}

With the notation and setting introduced in the previous sections, we
point out that the  ``inverse learning'' problem \eqref{eq:learnmodel}
can, provided Assumptions \eqref{assumpteval} and \eqref{assumptCarleman}
are met, be reduced to a classical learning problem 
(hereafter called ``direct'' learning) under the setting and assumptions of reproducing 
kernel based estimation methods.
In the direct learning setting, the model is
given by \eqref{eq:learnmodel} (i.e. $Y_i=g(X_i) + \eps_i$)  and
the goal is to estimate the function $g$. Kernel methods {\em posit} that $g$ belongs
to some reproducing kernel Hilbert space\footnote{This can be extended to the case where $g$
  is only approximated in $L^2(\nux)$ by a sequence of functions in $\hhh_K$.
  For the sake of the present discussion, only the case where it is assumed $g \in \hhh_K$
  is of interest.} 
$\hhh_K$ with kernel $K$ and 
construct an estimate $\wh{g} \in \hhh_K$ of $g$ based on the observed data.
The reconstruction error $(\wh{g}-g)$ can be analyzed in $L^2(\nux)$ norm
or in $\hhh_K$-norm.

Coming back to the inverse learning setting ($Y_i=(Af)(X_i) + \eps_i$), let $\hhh_K$ be defined
as in the previous sections and assume $f \in \ker(A)^\perp$ (we cannot hope to recover
the part of $f$ belonging to $\ker{A}$ anyway, and might as well make this assumption. It
is also implied by any form of source condition as introduced in Section~\ref{priors}).

Consider applying a direct learning method using the reproducing kernel $K$; this returns
some estimate $\wh{g} \in \hhh_K$ of $g$. Now defining $\wh{f}:=A^{-1}\wh{g}$\,, we have
\[
\norm{\wh{f}-f}^2_{\hhh_1} = \norm{ A^{-1}\wh{g}-f}^2_{\hhh_1} =
\norm{\wh{g}-Af}^2_{\hhh_K} = \norm{\wh{g}-g}^2_{\hhh_K} \,,
\]
by the partial isometry property of $A$ as an operator $\hhh_1\mapsto \hhh_K$
(Proposition~\ref{prop:rkhsprop}). Note that $\wh{f}$ is, at least in principle,
accessible to the statistician, since $A$ (and therefore $A^{-1}$) is assumed to
be known.
Hence, a bound established for the direct learning setting in the sense
of the $\hhh_K$-norm reconstruction $\norm{\wh{g}-g}^2_{\hhh_K}$ also applies
to the inverse problem reconstruction error $\norm{\wh{f}-f}^2_{\hhh_1}$\,.
Furthermore, it is easy to see that the eigenvalue decay conditions
and the source conditions involving the operator $B_\nux$ introduced in
Section~\ref{priors} are, via the same isometry,  equivalent
to similar conditions involving the kernel integral operator in the
direct learning setting, as considered for instance in 
\cite{per,optimalratesRLS,optimalrates,smalezhou2}. It follows that
estimates in $\hhh_K$-norm available from those references are directly
applicable to the inverse learning setting. However, as summarized in Table~\ref{ta:table1}, for the direct learning problem
the results concerning $\hhh_K$-norm rates of convergence are far less
complete than in $L^2(\nux)$-norm.
In particular, such rates have not been established under consideration of
simultaneous source and eigenvalue decay conditions, and neither have the
corresponding lower bounds. In this sense, the contribution of the present paper
is to complete the picture in Table~\ref{ta:table1},
with the inverse learning setting as the underlying motivation.


\subsection{Regularization}
\label{se:regulariz}

In this section, we introduce the class of linear regularization methods based on spectral theory for 
self-adjoint linear operators. These are standard methods for finding stable solutions for ill-posed inverse 
problems, see e.g. \cite{engl} or \cite{rosasco}. 

\begin{defi}[\protect{\bf Regularization function}]
\label{regudef}
Let $g: (0,1]\times [0, 1] \longrightarrow \R$ be a function and write 
$g_{\lam}=g(\lam, \cdot)$. The family $\{g_{\lam}\}_{\lam}$ is called 
{\it regularization function}, if the following conditions hold: 
\begin{enumerate}
\item[(i)]
There exists a constant $D<\infty$ such that
\begin{equation*}
\sup_{0<t\leq 1}|tg_{\lam}(t)| \leq D,
\end{equation*}
for any $0 < \lam \leq 1$.

\item[(ii)]
There exists a constant $E<\infty$ such that
\begin{equation}
  \label{eq:supg}
\sup_{0<t\leq 1}|g_{\lam}(t)| \leq \frac{E}{\lam},
\end{equation}
for any $0 < \lam \leq 1$.

\item[(iii)]
Defining the {\em residual} 
\begin{equation}
\label{rlamdef}
r_{\lam}(t)= 1-g_{\lam}(t)t\,,
\end{equation}
there exists a constant $\gamma_0 <\infty$ such that
\begin{equation*}
\sup_{0<t\leq 1}|r_{\lam}(t)| \leq \gamma_0 ,
\end{equation*}
for any $0 < \lam \leq 1$.
\end{enumerate}
\end{defi}

\begin{defi}[\protect{\bf Qualification}]

The {\it qualification} of the regularization $\{g_{\lam}\}_{\lam}$ is the maximal $q$ such that for any $0<\lam\leq 1$
\begin{equation*}
\sup_{0<t\leq 1}|r_{\lam}(t)|t^{q} \leq \gamma_{q}\lam^{q}.
\end{equation*}
for some constant $\gamma_q>0$\,.
\end{defi}

The next lemma provides a simple inequality (see e.g. \cite{mathe3}, Proposition 3 ) that shall be used later. 
\begin{lem}
\label{rsmallernu}{\rm
Let $\{g_{\lam}\}_{\lam}$ be a regularization function with qualification $q$. Then, for any 
$r \leq q$ and $0<\lam\leq 1$:
\begin{equation*}
\sup_{0<t\leq 1}|r_{\lam}(t)|t^r \leq \gamma_r \lam^r,  
\end{equation*}
where $\gamma_{r}:= \gamma_0^{1-\frac{r}{q}} \gamma_q^{\frac{r}{q}}$\,.
}
\end{lem}

We give some examples which are common both in classical inverse problems \cite{engl} and in learning theory \cite{per}.

\begin{example}{\bf (Spectral Cut-off) } 
A very classical regularization method is {\it spectral cut-off} (or truncated singular value decomposition), defined by
\[ g_{\lam}(t) = \left\{ \begin{array}{ll}
         \frac{1}{t} & \mbox{if $t \geq \lam$}  \\
        0 & \mbox{if $t < \lam$} \; .\end{array} \right. \] 
In this case, $D=E=\gamma_0 = \gamma_q=1$. The qualification q of this method can be arbitrary.        
\end{example}

\begin{example}{\bf (Tikhonov Regularization) } 
The choice $g_{\lam}(t) = \frac{1}{\lam + t}$ corresponds to {\it Tikhonov regularization}. In this case we have $D=E=\gamma_0 =1$. 
The qualification of this method is $q =1$ with $\gamma_{q} = 1$.
\end{example}

\begin{example}{\bf (Landweber Iteration) } 
The {\it Landweber Iteration} (gradient descent algorithm with constant stepsize) is defined by
\[ g_{k}(t) = \sum_{j=0}^{k - 1}(1-t)^j \, \mbox{ with $k=1/\lam$ $\in \N$} \; . \]
We have $D=E=\gamma_0 = 1$. The quailfication q of this algorithm can be arbitrary with $\gamma_q=1$ if $0<q\leq 1$ and $\gamma_q=q^q$ if $q>1$. 
\end{example}

Given the sample $\z=(\x, \y)\in (\X \times \R)^n$, we define the regularized approximate solution $f_{\z}^{\lam}$ of problem 
$(\ref{samplecase})$, for a suitable a-priori parameter choice $\lam = \lam_n$, by 
\begin{equation}
\label{estimator}
f_{\z}^{\lam_n}:=  g_{\lam_n}(\kappa^{-2} B_{\x})\kappa^{-2} S_{\x}^{\star}\y =
g_{\lam_n}(\bar B_{\x})\bar S_{\x}^{\star}\y \; ,
\end{equation}
where we have introduced the shortcut notation $\bar B_{x} := \kappa^{-2} B_{\x}$
and $ \bar S_{\x} := \kappa^{-2} S_{\x}$\,. Note that 
$g_\lam(\bar B_\x)$ is well defined since $\norm{\bar B_\x}\leq 1$\,.


\section{Main results: upper and lower bounds on convergence rates}

\label{sec:mainresults}


Before stating our main results, we recall some basic definitions in order to clarify what we mean by asymptotic upper rate,
lower rate and minimax rate optimality. We want to track the precise behavior of these rates not only for what concerns the
exponent in the number of examples $n$, but also in terms of their scaling
(multiplicative constant)
as a function of some important parameters (namely the noise variance $\sigma^2$ and the complexity radius $R$ in the source condition).
For this reason, we introduce a notion of a family of rates over a family of models.
More precisely, in all the forthcoming definitions, we consider an indexed family
$(\M_\theta)_{\theta \in \Theta}$\,, where for all $\theta \in \Theta$\,, $\M_\theta$ is a class of Borel probability distributions
on $\X \times \R$ satisfying the basic general assumption \ref{basicmodel}.
We consider rates of convergence in the sense of the $p$-th moments of the estimation error, where $p>0$ is a fixed real number.

\begin{defi}(Upper Rate of Convergence)\\
  A  family of sequences
  $(a_{n,\theta})_{(n,\theta) \in \N \times \Theta}$ of positive numbers is called {\bf upper rate of convergence in $L^p$} for the interpolation
norm of parameter $s\in[0,\frac{1}{2}]$\,, over
the family of models $(\M_\theta)_{\theta \in \Theta}$\,, for the sequence of estimated solutions $(f_{\z}^{\lam_{n,\theta}})_{(n,\theta) \in \N
\times \Theta}$\,, 
using regularization parameters  $(\lam_{n,\theta})_{(n,\theta) \in \N \times \Theta}$\,,  if 
\[ \sup_{\theta \in \Theta} \limsup_{n \to \infty} \sup_{\rho \in {\M_\theta}}  
\frac{\E_{\rho^{\otimes n}}\big[ \|B_\nux^s(f_{\rho} - f_{\z}^{\lam_{n,\theta}})\|_{\h}^p \big]^{\frac{1}{p}}}{a_{n,\theta}} < \infty \,.\]
\end{defi}

\begin{defi} (Weak and Strong Minimax Lower Rate of Convergence)\\
    A  family of sequences  $(a_{n,\theta})_{(n,\theta) \in \N \times \Theta}$
    of positive numbers is called {\bf weak minimax lower rate of convergence
    in $L^p$} for the interpolation
    norm of parameter $s\in[0,\frac{1}{2}]$\,, over the family of models $(\M_\theta)_{\theta \in \Theta}$\,, if
\[ \inf_{\theta \in \Theta} \limsup_{n \to \infty} \inf_{f_{\bullet}}\sup_{\rho \in {\M_\theta}}
     \frac{\E_{\rho^{\otimes n}} \brac{\norm{B_\nux^s(f_{\rho} - f_{\z})}_{\h}^p}^{\frac{1}{p}}}{a_{n,\theta}}  > 0 \,, \]   
     where the infimum is taken over all estimators, i.e. measurable mappings $f_{\bullet}: (\X\times\R)^n \pfeil \h$\,.
          It is called a {\bf strong minimax lower rate of convergence in $L^p$} if
     \[ \inf_{\theta \in \Theta} \liminf_{n \to \infty} \inf_{f_{\bullet}}\sup_{\rho \in {\M_\theta}}
     \frac{\E_{\rho^{\otimes n}} \brac{\norm{B_\nux^s(f_{\rho} - f_{\z})}^p_{\h}}^{\frac{1}{p}}}{a_{n,\theta}}  > 0 \,. \]   
\end{defi}

The difference between weak and strong lower rate can be summarily reformulated in the following
way: if $r_n$ denotes the sequence of minimax errors for a given model and reconstruction error,
using $n$ observations, then $a_n = \mathcal{O}(r_n)$ must hold if $a_n$ is a strong lower rate,
while $a_n$ being a weak lower means that $r_n = o(a_n)$ is excluded.

\begin{defi} (Minimax Optimal Rate of Convergence)\\
The sequence of estimated solutions $(f_{\z}^{\lam_{n,\theta}})_n$ 
using the regularization parameters $(\lam_{n,\theta})_{(n,\theta)\in \N\times\Theta}$ is called {\bf weak/strong minimax optimal in $L^p$}
for the interpolation norm of parameter $s\in[0,\frac{1}{2}]$,
over the model family  $(\M_\theta)_{\theta \in \Theta}$, with {\bf rate of convergence} given by the sequence $(a_{n_\theta})_{(n,\theta) \in \N \times \Theta}$, if the latter is a weak/strong minimax lower rate as well as an upper rate for $(f_{\z}^{\lam_{n,\theta}})_{n,\theta}$.
\end{defi}

We now formulate our main theorems.

\begin{theo}
\label{maintheo3}
Consider the model $\M_{\sigma,M,R}:={\M}(r,R, {\priorle}(b, \beta))$\,
(as defined in Section~\ref{priors}),
where $r>0$, $b>1$ and $\beta>0$ are fixed, and $(R,M,\sigma) \in \R^3_+$ (remember that $(\sigma,M)$
are the parameters in the Bernstein moment condition \eqref{bernstein}, in particular $\sigma^2$ is a bound
on the noise variance.)
Given a sample $\z=(\x, \y)\in (\X \times \R)^n$, define $f_{\z}^{\lam}$ as in $(\ref{estimator})$,
using a regularization function of qualification $q\geq r+s$, with the parameter sequence
\begin{equation}
  \label{paramrule}
   \lam_{n,(\sigma,R)} = \min\paren{ \paren{ \frac{\sigma^2}{R^2 n}}^{\frac{b}{2br+b + 1}},1} \; .
\end{equation}
Then for any $s\in[0,\frac{1}{2}]$, the sequence  
\begin{equation}
\label{rateseq}
 a_{n,(\sigma,R)} = R \paren{\frac{\sigma^2}{R^2n}}^{\frac{b(r+s)}{2br+b+1}} 
\end{equation}
is an upper rate of convergence in $L^p$ for all $p>0$, for the interpolation
norm of parameter $s$, for the sequence of estimated solutions $(f_{\z}^{\lam_{n,(\sigma,R)}})$ over 
the family of models $(\M_{\sigma,M,R})_{(\sigma,M,R) \in \R_+^3}$\,.
\end{theo}

\begin{theo}
\label{minimaxlowerrate}
Let $r>0,R>0,b>1$ and $\alpha >0$ be fixed. 
Let $\nux$ be a distribution on $\X$ belonging to $\priorgr(b, \alpha)$.
Then the sequence  $(a_{n,(\sigma,R)})$ defined in \eqref{rateseq} is a weak minimax lower rate of convergence in $L^p$ for all $p>0$\, , for the
model family $\M_{R,M,\sigma} := \M(r,R,\set{\nu})$\,, $(R,M,\sigma) \in \R_+^3$\,.
If $\nux$ belongs to $\priorgr_{strong}(b, \alpha)$, then the sequence  $a_{n,(\sigma,R)}$ is a strong minimax lower rate of convergence in $L^p$ for all $p>0$\,,
for the model family $\M_{R,M,\sigma}$\,.
\end{theo}

Finally, we have as a direct consequence:

\begin{cor}
\label{maincor2}
Let $r>0, b>1$, $\beta\geq \alpha >0$ be fixed and assume
$\PPP'=\priorle(b, \beta) \cap {\priorgr}(b, \alpha) \neq \emptyset$\,.
Then the sequence of estimators $f_{\z}^{\lam_{n,(\sigma,R)}}$ as defined in \eqref{estimator} is 
strong minimax optimal in $L^p$ for all $p>0$,  
under the assumptions and parameter sequence \eqref{paramrule} of Theorem~\ref{maintheo3}\,,
over the class $\M_{R,M,\sigma} := \M(r,R,\PPP')$\,,  $(R,M,\sigma) \in \R_+^3$\,.
\end{cor}
 


\section{Discussion}

\label{se:disc}

We conclude by briefly discusssing some specific points related to our results.

{\em Non-asymptotic, high probability bounds.} The results presented in Section~\ref{sec:mainresults}
are asymptotic in nature and concern moments of the reconstruction error.
However, the main underlying technical result is an exponential deviation inequality which holds non-asymptotically. For simplicity of the exposition we have chose to relegate this result to the
Appendix (Proposition~\ref{maintheo1} there). Clearly, this is thanks to such a  deviation inequality
that we are able to handle moments of all orders of the error. Furthermore, while the asymptotics considered
in the previous section always assume that all parameters are fixed as $n\rightarrow \infty$\,,
going back to the deviation inequality one could in principle analyze asymptotics of other nonstandard regimes where some parameters are allowed to depend on $n$\,.

{\em Adaptivity.} For our results we have assumed that the crucial parameters $b,r,R$ concerning the eigenvalue decay of the operator $B_\nu$ as well as the regularity of the target function are known, and so is the noise variance $\sigma$\,; these parameters are used in the choice of regulatizing constant $\lambda_n$\,.
This is, of course, very unrealistic. Ideally, we would like to have a procedure doing almost as good
without knowledge of these parameters in advance -- this is the question of adaptivity. While this topic
is outside of the scope of the present paper, in work in progress we study such an adaptive procedure
based on Lepski's principle for the oracle selection of a suitable regularizing constant $\lambda$ --
this is again a situation where an exponential deviation inequality is a particularly relevant tool.

{\em Weak and strong lower bounds.} The notion of strong and weak lower bounds introduced in this work
(corresponding respectively to a lim inf and lim sup in $n$) appear to be new. They were motivated by
the goal to consider somewhat minimal assumptions on the eigenvalue behavior, i.e. only a one-sided
power decay bound, to obtain lower minimax bounds under source condition regularity. It turns out
a one-sided power decay bound is the main driver for minimax rates, but excluding arbitrarly abrupt
relative variations $\eigv_{2j}/\eigv_j$ appears to play a role in distinguishing the weak and strong
versions. Such a condition is also called one-sided regular variation, see \cite{regvar} for extensive considerations on such issues. We believe that this type of assumption can be relevant for the
analysis of certain
inverse problems when the eigenvalues do not exhibit a two-sided power decay.

{\em Smoothness and source conditions.} In considering source conditions \eqref{sourceset} in
terms of the operator $B_\nu$ as measure of regularity of the target $f$\,, 
we have followed the general approach adopted in previous works on
statistical learning using kernels, itself inspired by the setting considered in the
(deterministic) inverse problem literature. It is well-established in the latter literature that
representing the target function in terms of powers of the operator to be inverted is a very natural
way to measure its regularity; it can be seen as a way to relate noise and signal in a geometry that is
appropriate for the considered ill-posed problem.
In our setting, one can however wonder why a measure of regularity of the
target function should depend on the sampling distribution $\nu$\,. A high-level answer is that the sampling
can itself be seen as a source of noise (or uncertainty), and that it is natural that it enters in the
ill-posedness of the problem. For instance, regions in space with sparser sampling will result in more
uncertainty. On the other hand, if, say, the support of $\nu$ is contained in a low-dimensional manifold,
the problem becomes intrinsically lower-dimensional, being understood that we must abandon any hope of estimating outside of the support, and this should also be reflected in the measure of regularity.
A more detailed analysis of such issues, and relations to more common notions of regularity, is out of the
scope of the present work but certainly an interesting future perspective.


\section{Proof of Upper Rate}

\label{se:proofupper}

We recall the shortcut notation $\bar B_{x} := \kappa^{-2} B_{\x}$\,,
$ \bar S_{\x} := \kappa^{-2} S_{\x}$\, and similarly
define $\bar B := \kappa^{-2} B$\,.
Recall that we denote $\norm{A}$ the spectral norm of an operator $A$
between Hilbert spaces; additionally we will denote $\norm{A}_\hs$ the
Hilbert-Schmidt norm of $A$ (assuming it is well-defined).

All along the proof, we will use the notation $C_a$ to denote a positive factor
only depending on the quantity $a$. The exact expression of this factor
depends on the context and can potentially change from line to line.

\subsection{Concentration Inequalities}

We introduce the {\it effective dimension} $\NN (\lam)$,
appearing in \cite{optimalratesRLS} in a  similar context.
For $\lambda \in (0,1]$ we set
\begin{equation}
\label{effectivedim}
\NN (\lam) = \tr(\;(\bar B+ \lam)^{-1} \bar B\;) \;.
\end{equation}  
Since by Proposition $\ref{Carlemanoperator}$ the operator $B$ is trace-class,
$\NN(\lam)< \infty$. Moreover, we have the following estimate (see \cite{optimalratesRLS}, Proposition 3):
\begin{lem}
  Assume that the marginal distribution $\nu$ of $X$ belongs
  to $\priorle(b,\beta)$ (with $b>1$ and $\beta>0$).
    Then the effective dimension $\NN(\lam)$ satisfies
    \[ \NN(\lam)\leq  \frac{\beta b}{b-1} (\kappa^2\lam) ^{-\frac{1}{b}}\; .
    \]
\end{lem}
Furthermore, for $\lam \leq ||\bar B||$, since $\bar B$ is positive 
\begin{equation}
\label{effdimlow}
\NN(\lam)   = \sum_{\eigv_j \geq \kappa^2\lam} \frac{\eigv_j}{\eigv_j + \kappa^2\lam} + \sum_{\eigv_j < \kappa^2\lam} \frac{\eigv_j}{\eigv_j + \kappa^2\lam} \nonumber
            \geq   \min_{\eigv_j \geq \kappa^2\lam} \left\{ \frac{\eigv_j}{\eigv_j + \kappa^2\lam} \right\} \geq \frac{1}{2} \; ,   
\end{equation}
since the first sum has at least one term.
The following propositions summarize important concentration properties of the
empirical quantities involved. The proofs are given in Appendix \ref{app:concentration}.

\begin{prop}
\label{Geta1}  
For $n \in \N$, $\lambda \in (0,1]$ and  $\eta \in (0,1]$, it holds with probability at least $1-\eta$\,:
\begin{equation*}
\big\| (\bar B +  \lam)^{-\frac{1}{2}}\;\left(\bar B_{\x}f_{\rho} - \bar S_{\x}^{\star}\y \right)\big\|_{\h}\;  \leq \; 
2\log(2\eta^{-1}) \kappa^{-1} \left( \frac{M}{n\sqrt{\lam}} + \sqrt{\frac{\sigma^2 {\cal N}(\lam)}{ n}} \right)\;.
\end{equation*}
Also, it holds with probability at least $1-\eta$:
\[
\norm{\bar B_{\x}f_{\rho} - \bar S_{\x}^{\star}\y }_{\h}\;  \leq \; 
2\log(2\eta^{-1}) \kappa^{-1} \left( \frac{M}{n} + \sqrt{\frac{\sigma^2}{ n}} \right)\;.
\]
\end{prop}

\begin{prop}
\label{Geta2}
For any $n \in \N$, $\lambda \in (0,1]$ and $\eta \in (0,1)$, it holds with probability at least $1-\eta $\,:
\begin{equation*}
\norm{(\bar B+ \lam)^{-1}(\bar B- \bar B_{\x}) }_{\hs} \; \leq 
2\log(2\eta^{-1}) \left( \frac{2}{n \lam} + \sqrt{\frac{\cal N(\lam)}{n\lam }}  \right)\; .
\end{equation*}
\end{prop}

\begin{prop}
\label{neumann}
Let $\eta \in (0,1)$. Assume that $\lambda \in (0,1]$ satisfies
\begin{equation}
\label{assumpt2b}
\sqrt{n\lam} \geq  8\log(2\eta^{-1}) \sqrt{\max({\cal N}(\lam),1)}\; .
\end{equation}
Then, with probability at least $1-\eta$\,:
\begin{equation}
\label{eq:mult_ineq}
\norm{ (\bar B_{\x} + \lam)^{-1}(\bar B + \lam) }\leq   2 \;.
\end{equation}
\end{prop}

\begin{prop}
\label{rem1}
For any $n \in \N$ and $0<\eta<1$ it holds with probability at least $1-\eta$\,: 
\begin{equation*}
\norm{ \bar B- \bar B_{\x} }_{\hs} \; \leq \; 6\log(2\eta^{-1})\; \frac{1}{\sqrt n} \; .
\end{equation*}
\end{prop}


\subsection{Some operator perturbation inequalities}

\begin{prop}
\label{optikh}
Let $B_1, B_2$ be two non-negative self-adjoint operators on some Hilbert space with 
$\norm{B_j}\leq a$, $j=1,2$, for some non-negative $a$. 
\begin{enumerate}
\item[(i)]
If $0\leq r \leq 1$, then
\[ \norm{ B_1^r - B_2^r} \leq C_r \norm{B_1-B_2}^r , \]
for some $C_r< \infty$.
\item[(ii)]
If $r>1$, then
\begin{equation*}
\norm{B_1^r-B_2^r} \leq C_{a,r} \norm{ B_1 - B_2 }  ,
\end{equation*}
for some $C_{a,r}<\infty$.
\end{enumerate}
\end{prop}

\begin{proof}
\begin{enumerate}
\item[(i)]
Since $t \mapsto t^r$ is operator monotone if $r \in [0,1]$, this result follows from Theorem X.1.1 in \cite{Bat97} for positive matrices, 
but it the proof applies as well to positive operators on a Hilbert space.  
\item[(ii)]
The proof follows from Proposition \ref{app:pert:prop} in Appendix \ref{app:perturbation}.
\end{enumerate}
\end{proof}

\begin{prop}[\cite{Bat97}, Theorem IX.2.1-2]
\label{prop:Bath-ineq}
Let $A,B$ be to self-adjoint, positive operators on a Hilbert space. Then for any $s\in[0,1]$:
\begin{equation}
\label{eq:multpert}
\norm{A^sB^s} \leq \norm{AB}^s\,.
\end{equation}
\end{prop}
Note: this result is stated for positive matrices
in \cite{Bat97}, but it is easy to check that the proof applies
as well to positive operators on a Hilbert space.



\subsection{Proof of Theorem \ref{maintheo3}}
The following proposition is our main error bound and the convergence rate will follow.
\begin{prop}
\label{maintheo1}
Let $s\in[0,\frac{1}{2}]$, $r>0$, $R>0$, $M>0$\,.
Suppose $\fo \in \Omega_\nu(r, R)$ (defined in $(\ref{sourceset})$)\,.
Let $f_{\z}^{\lam}$ be defined as in \eqref{estimator} using
a regularization function of qualification
$q \geq r+s$ and put $\bar \gamma := \max(\gamma_0,\gamma_q)$\,.
Then, for any $\eta \in (0,1)$, $\lam \in (0, 1]$ and $n \in \N$ satisfying
\begin{equation}
\label{assumpt1}
  n \geq 64 \lambda^{-1} \max({\cal N}(\lam),1) \log^2{\left(8 / \eta\right)}  \; ,
\end{equation}  
we have  with probability at least 
$1-\eta$:
\begin{equation}
\label{maintheo1eq}
 \norm{B^s (\fo - f_{\z}^{\lam})}_{\h}  \leq C_{r,s,D,E} \; \log(8\eta^{-1}) \kappa^{2s} \lam^s \left( \bar \gamma \kappa^{2r} R\paren{\lam^r +   \frac{1}{\sqrt n}}  
+ \kappa^{-1} \left( \frac{M}{n\lambda} + \sqrt{\frac{\sigma^2 {{\cal N}(\lam)}}{{n \lambda}}}\right) \right)\,.
\end{equation}
\end{prop}

\begin{proof}
We start with a preliminary inequality. Assumption \eqref{assumpt1}
implies that \eqref{assumpt2b} holds with $\eta$ replaced by $\eta/4$\,.
We can therefore apply Proposition \ref{neumann} and obtain
that, with probability at least $1-\eta/4$, inequality \eqref{eq:mult_ineq} holds. Combining this with \eqref{eq:multpert}, we
get for any $u \in [0,1]$:
\begin{equation}
\label{eq:prelim}
\norm{\bar B^u(\bar B_{\x} + \lambda )^{-u}} = \norm{\bar B^u(\bar B + \lambda )^{-u} (\bar B +  \lambda )^u (\bar B_{\x} +  \lambda )^{-u}}
\leq \norm{(\bar B + \lambda ) (\bar B_{\x} +  \lambda )^{-1}}^u \leq 2.
\end{equation}
From this we deduce readily that, with probability at least $1-\eta/4$, we have
\begin{equation}
\label{eq:prem1}
\norm{B^s(\fo - f_{\z}^{\lam} )}_{\h} \leq 2 \kappa^{2s} \norm{(\bar  B_{\x} + \lambda )^{s}(\fo - f_{\z}^{\lam} )}_{\h}. 
\end{equation}
We now consider the following decomposition 
\begin{align}
 \fo - f_{\z}^{\lam} & =  \fo - g_{\lam}(\bar B_{\x})\bar S_{\x}^{\star}\y  \nonumber \\  
                   & =  (\fo - g_{\lam}(\bar B_{\x})\bar B_{\x}f) + g_{\lam}(\bar  B_{\x})(\bar B_{\x}f - \bar S_{\x}^{\star}\y ) \nonumber \\
                   & =  r_{\lam}(\bar B_{\x})f + g_{\lam}(\bar B_{\x})(\bar B_{\x}\fo - \bar S_{\x}^{\star}\y )\,, \label{split}
\end{align}
where $r_{\lam}$ is given by $(\ref{rlamdef})$. We now upper bound $\norm{(\bar B_{\x} + \lambda )^{s}(\fo - f_{\z}^{\lam} )}_{\h}$ by
treating separately the two terms corresponding to the above decomposition.

{\bf Step $1$:} First term: since $\fo \in \Omega_\nu(r,R)$, we have
\[
\norm{ (\bar B_{\x} + \lambda)^{s} r_{\lam}(\bar B_{\x})\fo  }_{\h} \leq
\kappa^{2r} R\; \norm{ (\bar B_{\x} + \lambda)^{s} r_{\lam}(\bar B_{\x})\bar B^r} . \] 
We now concentrate on the operator norm appearing in the RHS of the above bound, and distinguish
between two cases. The first case is $r\geq 1$, for which we write
\begin{equation}
\label{furthersplit}
(\bar B_{\x} + \lambda)^{s} r_{\lam}(\bar B_{\x})\bar B^r  =  (\bar B_{\x} + \lambda)^{s}
r_{\lam}(\bar B_{\x})\bar B_{\x}^r + 
(\bar B_{\x} + \lambda)^{s}   r_{\lam}(\bar B_{\x})
( \bar B^r - \bar B_{\x}^r ).                           
\end{equation}
The operator norm of the first term is estimated via
\begin{align}
\norm{(\bar B_{\x} + \lambda)^{s}
r_{\lam}(\bar B_{\x})\bar B_{\x}^r} & 
\leq \sup_{t \in [0,1]} (t+\lambda)^st^rr_\lam(t) \nonumber \\
&\leq \sup_{t \in [0,1]} t^{s+r} r_\lam(t) + \lambda^s \sup_{t \in [0,1]} t^{r} r_\lam(t) 
\nonumber \\
& \leq 2\bar \gamma \lambda^{s+r}, \label{eq:device}
\end{align}
by applying (twice) Lemma~\ref{rsmallernu} and the assumption that the qualification $q$ of the regularization is greater than $r+s$\,; we also
introduced $\bar \gamma := \max(\gamma_0,\gamma_{q})$\,.
The second term in equation \eqref{furthersplit} is estimated via
\[
\norm{(\bar B_{\x} + \lambda)^{s}   r_{\lam}(\bar B_{\x})( \bar B^r - \bar B_{\x}^r )}
\leq \norm{(\bar B_{\x} + \lambda)^{s}   r_{\lam}(\bar B_{\x})} \; \norm{ \bar B^r - \bar B_{\x}^r }
\leq 2 \bar \gamma C_r \;\lambda^s \; \norm{\bar B - \bar B_{\x}}.
\]
For the first factor we have used the same device as previously for the first term
based on Lemma~\ref{rsmallernu}, and for the second factor we used Proposition \ref{optikh} (ii).
Finally using Proposition $\ref{rem1}$ to upper bound $\norm{\bar B - \bar B_{\x}}$, 
collecting the previous estimates we obtain with probability at least $1-\eta/2$:
\begin{equation}
\label{eq:term1final}
\norm{ (\bar B_{\x} + \lambda)^{s} r_{\lam}(\bar B_{\x})\fo }_{\h}
\leq \bar \gamma  C_r \kappa^{2r} R \log \paren{4\eta^{-1} } \paren{
\lambda^r + \frac{1}{\sqrt{n}}} \lambda^s\,.
\end{equation}
We turn to the case $r<1$, for which we want to establish a similar inequality.
Instead of \eqref{furthersplit} we use:
\begin{align*}
\norm{(\bar B_{\x} + \lambda)^{s} r_{\lam}(\bar B_{\x})\bar B^r} & =
\norm{(\bar B_{\x} + \lambda)^{s} r_{\lam}(\bar B_{\x}) (\bar B_{\x} + \lambda)^r (\bar B_{\x} + \lambda)^{-r} \bar B^r }\\
& \leq 2 \norm{(\bar B_{\x} + \lambda)^{r+s} r_{\lam}(\bar B_{\x})} \\
& \leq 8 \bar \gamma \lambda^{r+s},
\end{align*}
where we have used the (transposed version of) inequality \eqref{eq:prelim} (valid
with probability at least $1-\eta/4$);
and, for the last inequality, an argument similar the one leading to 
\eqref{eq:device}
(using this time that $(t+\lambda)^{r+s} \leq 2 (t^{r+s} + \lambda^{r+s})$ for all $t\geq 0$
since $r+s \leq 2$ in the case we are considering). This implies {\em a fortiori}, that inequality \eqref{eq:term1final} holds in the case $r<1$ as well (also with probability
at least $1-\eta/2$).

{\bf Step $2$:} Bound on $\norm{ (\bar B_{\x} + \lambda)^s g_{\lam}(\bar B_{\x})(\bar B_{\x}\fo - \bar S_{\x}^{\star}\y )}_{\h}$.

We further split by writing
\begin{eqnarray}
\label{H123}
(\bar B_{\x} + \lam)^s g_{\lam}(\bar B_{\x})(\bar B_{\x}\fo - \bar S_{\x}^{\star}\y ) & = & H_{\x}^{(1)}\cdot H_{\x}^{(2)} \cdot h^{\lam}_{\z}
\end{eqnarray}
with
\begin{eqnarray*}
H_{\x}^{(1)} &:=& (\bar B_{\x} + \lam)^{s} g_{\lam}(\bar B_{\x})(\bar B_{\x} + \lam)^{\frac{1}{2}} ,\\
H_{\x}^{(2)} &:=& (\bar B_{\x} + \lam)^{-\frac{1}{2}}(\bar B + \lam)^{\frac{1}{2}} , \\
h^{\lam}_{\z} &:=& (\bar B + \lam)^{-\frac{1}{2}} (\bar B_{\x}\fo - \bar S_{\x}^{\star}\y ) 
\end{eqnarray*}
and proceed by bounding each factor separately. 

For the first term, we have (for any $\lambda \in (0,1]$ and $\x \in \X^n$),
and remembering that $s \leq 1/2$:
\begin{align}
 \norm{\; H_{\x}^{(1)} \;} & \leq \sup_{t \in [0,1]} (t+\lam)^{s+\frac{1}{2}} g_\lam(t) \nonumber \\
& \leq \lam^{s+\frac{1}{2}} \sup_{t \in [0,1]}  g_\lam(t) + 
\sup_{t \in [0,1]} \abs{t^{s+\frac{1}{2}} g_\lam(t)} \nonumber \\
& \leq E \lam^{s-\frac{1}{2}} + \paren{\sup_{t \in [0,1]} \abs{t g_\lambda(t)}}^{s + \frac{1}{2}}
\paren{\sup_{t \in [0,1]} \abs{g_\lam(t)}}^{\frac{1}{2}-s} \nonumber \\
& \leq E \lam^{s-\frac{1}{2}} + D^{s+\frac{1}{2}} E^{\frac{1}{2}-s} \lambda^{s-\frac{1}{2}}
= C_{s,D,E} \lam^{s-\frac{1}{2}}, \label{H1}
\end{align}
where we have used Definition $\ref{regudef}$ $(i)$, $(ii)$.

The probabilistic bound on $H_{\x}^{(2)}$ follows from Proposition \ref{neumann},
which we can apply using assumption \eqref{assumpt1}, combined
with Proposition \ref{prop:Bath-ineq}. This ensures with probability at least $1-\eta/4$
\begin{equation}
\label{H2}
\norm{H_{\x}^{(2)}} \leq  2 \; .
\end{equation}
Finally, the probabilistic bound on $h^{\lam}_{\z}$ follows from Proposition $\ref{Geta1}$: with probability at least $1-\eta/4$, we have
\begin{equation}
\label{H3}
\norm{ h^{\lam}_{\z} }_{\h} \leq  2\log(8\eta^{-1}) \kappa^{-1} \left( \frac{M}{n\sqrt{\lambda}} + \sqrt{\frac{\sigma^2 {{\cal N}(\lam)}}{{n}}}\right)\,. 
\end{equation}

As a result, combining $(\ref{H1})$, $(\ref{H2})$ and $(\ref{H3})$ with $(\ref{H123})$ gives with probability at least $1-\eta/2$
\begin{equation}
\label{second2}
\norm{(B_{\x} + \lambda)^s g_{\lam}(B_{\x})(B_{\x}f - S_{\x}^{\star}\y )}_{\h}  \leq  C_{s,D,E} \log(8\eta^{-1}) 
\kappa^{-1} \lambda^s \left( \frac{M}{n\lambda} + \sqrt{\frac{\sigma^2 {{\cal N}(\lam)}}{{n \lambda}}}\right)
\; ,
\end{equation}

We end the proof by collecting \eqref{eq:prem1}, \eqref{split}, \eqref{eq:term1final} and \eqref{second2} and finally obtain bound \eqref{maintheo1eq} holding with probability at least $1-\eta$ \,.
\end{proof}


\begin{cor}
\label{maincor}
Let $s\in[0,\frac{1}{2}]$, $\sigma>0, M>0, r>0,R>0,\beta>0,b>1$ and
assume the generating distribution of $(X,Y)$ belongs to ${\M}(r,R, {\priorle}(b, \beta))$
(defined in Section~\ref{priors})\,.
Let $f_{\z}^{\lam}$ be the estimator defined as in \eqref{estimator} using
a regularization function of qualification
$q \geq r+s$ and put $\bar \gamma := \max(\gamma_0,\gamma_q)$\,.
Then, there exists $n_0>0$ (depending on the above parameters), so that for all $n\geq n_0$,
if we set 
\begin{equation}
\label{eq:choicelam}
 \lam_n = \min\paren{ \paren{ \frac{\sigma}{R\sqrt n}}^{\frac{2b}{2br+b + 1}},1} \; ,
\end{equation}
then with probability at least $1-\eta$\,:
\[ \norm{B^s(f - f_{\z}^{\lam_n})}_{\h}  \leq   C_{r,s,b,\beta,\ol{\gamma},D,E,\kappa} \; \log(8\eta^{-1})
R \paren{\frac{\sigma}{R\sqrt{n}}}^{\frac{2b(r+s)}{2br+b+1}}\,,
\]
provided $\log \eta^{-1} \leq C_{b,\beta,\kappa,\sigma,R} n^{\frac{br}{2br+b+1}}$\,. 
\end{cor}

{\bf Remark:} In the above corollary, $n_0$ can possibly depend on all parameters\,,
but the constant in front of the upper bound does not depend on $R,\sigma$, nor $M$\,. In this sense,
this result tracks precisely the effect of these important parameters on the scaling of the rate,
but remains asymptotic in nature: it cannot be applied if, say, $R,\sigma$ of $M$ also depend on $n$
(because the requirement $n\geq n_0$ might then lead to an impossiblity.) If some parameters are
allowed to change with $n$\,, one should go back to the nonasymptotic statement of Theorem \ref{maintheo1} for an analysis of the rates.

\begin{proof}
We check that the assumptions of Proposition \ref{maintheo1} are satisfied provided $n$ is big enough.
Concerning assumption \eqref{assumpt1}, let us recall that by Lemma~\ref{effectivedim}:
\begin{equation}
  \label{eq:boundn}
        {\cal N}(\lam) \leq C_{b,\beta,\kappa} \lam^{-1/b} \; ,
\end{equation}
for some $C_{b,\beta,\kappa} > 0$. Consequently, \eqref{assumpt1} is ensured by the sufficient condition
\begin{equation}\label{eq:condnn}
  n \geq  C_{b,\beta,\kappa} \log^2(8\eta^{-1}) \lam^{-\frac{1}{b}-1}
  \;\Leftarrow \; \log \eta^{-1} \leq C_{b,\beta,\kappa,\sigma,R} n^{\frac{br}{2br+b+1}} \,.
  \end{equation}
Applying Proposition \ref{maintheo1}, Lemma~\ref{effectivedim} again, and folding the effect of
the parameters we do not intend to track precisely into a generic multiplicative constant, we obtain
that with probability $1-\eta$:
\begin{equation}
\label{maintheo1eqredux}
 \norm{B^s (\fo - f_{\z}^{\lam})}_{\h}  \leq C_{r,s,\kappa,\bar \gamma,D,E,b,\beta} \; \log(8\eta^{-1}) \lam^s \left( R\paren{\lam^r +   \frac{1}{\sqrt n}}  
+ \left( \frac{M}{n\lambda} + \frac{\sigma}{\sqrt{n}} \lambda^{-\frac{b+1}{2b}} \right)\right)\,.
\end{equation}

 Observe that the choice \eqref{eq:choicelam}
implies that $n^{-\frac{1}{2}} = o(\lambda_n^r)$. Therefore, up to requiring $n$ large enough
and multiplying the front factor by 2 ,
we can disregard the term  $1/\sqrt{n}$ in the second factor of the above bound. Similarly, by comparing
the exponents in $n$, one can readily check that
\[
\frac{M}{n\lam_n} = o\paren{\sqrt{\frac{1}{n}\lam_n^{-\frac{b+1}{b}}}}\,,
\]
so that we can also disregard the term $(n\lambda_n)^{-1}$ for $n$ large enough (again, up
multiplying the front factor by 2) and 
concentrate on the two remaining main terms of the upper bound in \eqref{maintheo1eqredux}, 
which are $R \lambda^r$ and $\sigma\lam^{-\frac{b+1}{2b}} n^{-\frac{1}{2}}$ \,.
The proposed choice of $\lambda_n$ balances precisely these two terms and 
easy computations lead to the announced conclusion.
\end{proof}

\begin{proof}[Proof of Theorem~\ref{maintheo3}]
We would like to ``integrate'' the bound of Corollary~\ref{maincor} over $\eta$
to obtain a bound in $L^p$ norm (see Lemma~\ref{le:boring} in the Appendix),
unfortunately the condition on $\eta$ prevents this
since very large deviations are excluded. To alleviate this, we first derive
a much coarser ``fallback'' upper bound which will be valid for all $\eta \in (0,1)$.
To this aim, we revisit shortly the proof of Proposition~\ref{maintheo1}\,.
We recall the decomposition (\ref{split})
\[
B^s( \fo - f_{\z}^{\lam}) =  B^s \big( r_{\lam}(\bar B_{\x})f + g_{\lam}(\bar B_{\x})(\bar B_{\x}\fo - \bar S_{\x}^{\star}\y )\big) \,.
 \]
 A rough bound on the first term using \eqref{rlamdef}, and $\fo \in \Omega(r,R)$ is
 \begin{equation}
   \label{eq:tt1}
   \norm{B^s r_{\lam}(\bar B_{\x})\fo}_{\h} \leq \gamma_0 \kappa^{2(r+s)} R\,.
 \end{equation}
 For the second term, using \eqref{eq:supg} and the second part or Proposition~\ref{Geta1}\,,  we obtain that with probability at least $1-\eta$\,,
 \begin{equation}
   \label{eq:tt2}
   \norm{ B^s  g_{\lam}(\bar B_{\x})(\bar B_{\x}\fo - \bar S_{\x}^{\star}\y )}_{\h}
   \leq 2 \kappa^{2s-1} \log (2\eta^{-1}) \frac{E}{\lambda}
   \paren{\frac{M}{n} + \sqrt{\frac{\sigma^2}{ n}} } \leq C_{\kappa,E,M,\sigma}
   \frac{1}{\lambda \sqrt{n}}\log (2 \eta^{-1})\,.
 \end{equation}
 For the rest of this proof, to simplify notation and argument
 we will adopt the following conventions:
 \begin{itemize}
 \item the dependence of multiplicative constants $C$ on various parameters will
   (generally) be omitted, except for $\sigma, M, R,\eta$ and $n$ which we want to track precisely\,.
 \item the expression ``for $n$ big enough'' means that the statement holds for
   $n\geq n_0$\,, with $n_0$ potentially depending on all model parameters
   (including $\sigma, M$ and $R$), but not on $\eta$\,.
 \end{itemize}
 From \eqref{eq:tt1} and \eqref{eq:tt2}\,, we conclude that
 \[
 \PP\Big[ \norm{B^s( \fo - f_{\z}^{\lam})}_{\h} \geq a' + b' \log \eta^{-1} \Big]
 \leq \eta\,,
 \]
 for all $\eta\in(0,1)$\,, with $a':= C_{\sigma,M,R}\max\paren{\frac{1}{\lambda \sqrt{n}},1}$ and $b':= \frac{C_{\sigma,M}}{\lambda \sqrt{n}}$\,.
 On the other hand, Corollary~\ref{maincor2}, ensured that
 \[
 \PP\Big[ \big\| B^s( \fo - f_{\z}^{\lam_{n,(\sigma,R)}}) \big\|_{\h} \geq a + b \log \eta^{-1} \Big]
 \leq \eta\,, \text{ for }  \log \eta^{-1} \leq \log \eta_0^{-1} := C_{\sigma,R} n^{\frac{br}{2br+b+1}}\,,
 \]
 with $a=b:=C R \paren{\frac{\sigma}{R\sqrt{n}}}^{\frac{2b(r+s)}{2br+b+1}}= C a_{n,(\sigma,R)}\,$\,,
 provided that $n$ is big enough.

 We can now apply Corollary~\ref{cor:boring} in the appendix, which encapsulates some tedious
 computations, to conclude that for any $p\leq \frac{1}{2} \log \eta_0^{-1}$\,,
 and $n$ big enough:
 \[
   \E\Big[\big\| B^s( \fo - f_{\z}^{\lam_{n,(\sigma,R)}})\big\|^p_{\h}\Big] \leq C_p \paren{a^p_{n,(\sigma,R)} + \eta_0 \paren{ (a')^p + 2(b' \log \eta_0^{-1})^p }}\,.
   \]
   Now for fixed $\sigma,M,R$\,, and $p$\,, the quantities $a',b'$ are
   powers of $n$\,, while $\eta_0 = \exp(-C_{\sigma,r} n^{\nu_{b,r}})$ for $\nu_{b,r}>0$\,.
   The condition $p\leq \frac{1}{2} \log \eta_0^{-1}$ is thus satisfied for $n$ large enough
   and we have
   \[
   \limsup_{n \rightarrow \infty} \sup_{\rho \in \mathcal{M}_{\sigma,M,R}} \frac{   \E_{\rho^{\otimes n}}\Big[\big\|B^s( \fo - f_{\z}^{\lam_{n,(\sigma,R)}})\big\|^p_{\h}\Big]^{\frac{1}{p}}}{a_{n,(\sigma,R)}} \leq C\,,
     \]
     (where we reiterate that the constant $C$ above may depend on all parameters including $p$\,,
     but not on $\sigma,M$ nor $R$.). Therefore taking the supremum over $(\sigma,M,R)$ yields
     the desired conclusion.
\end{proof}


\section{Proof of Lower Rate}

\label{se:prooflower}

Consider a model ${\cal P}=\{P_{\theta}: \theta \in \Theta\}$ of probability measures 
on a measurable space $(\Z, {\cal A})$\,, indexed by $\Theta$. Additionally, let $d: \Theta \times \Theta \longrightarrow [0, \infty)$ be a (semi-) distance. 

For two probability measures $P_1, P_2$ on some common measurable space $(\Z, {\cal A})$, we recall the definition of the {\it Kullback-Leibler divergence} 
between $P_1$ and $P_2$ 
\[ {\cal K}(P_1,P_2) := \int_{\X} \log\left( \frac{dP_1}{dP_2} \right) dP_1 \; ,\]
if $P_1$ is absolutely continuous with respect to $P_2$. If $P_1$ is not absolutely continuous with respect to $P_2$, 
then ${\cal K}(P_1,P_2) := \infty$. One easily observes that
\[  {\cal K}(P^{\otimes n}_1,P^{\otimes n}_2) \; = \; n \; {\cal K}(P_1,P_2) \;.\]

In order to obtain minimax lower bounds we briefly recall the general reduction scheme, presented in Chapter 2 of \cite{tsybakov}. 
The main idea is to find $N_{\eps}$ parameters $\theta_1, ..., \theta_{N_{\eps}} \in \Theta$, depending on $\eps < \eps_0$ for some $\eps_0>0$\,,
with $N_{\eps} \to \infty$ as $\eps \to 0$, 
such that any two of these parameters are $\eps$-separated with respect to the distance $d$,
but that the associated distributions $P_{\theta_j} =: P_j \in \PPP$ have small Kullback-Leibler divergence to each other
and are therefore statistically close.
It is then clear that 
 \begin{equation}
\label{reduction}
\inf_{\hat \theta}\sup_{\theta \in {\cal P}} \E_{\theta}[ d^p(\hat \theta, \theta)]^\frac{1}{p}
\geq \eps \inf_{\hat \theta}\sup_{\theta \in {\cal P}} \PP_{\theta}[ d(\hat \theta, \theta) \geq \eps ]
\geq \eps  \inf_{\hat \theta}\max_{1\leq j \leq N_{\eps}} \PP_{j}[ d(\hat \theta, \theta_j) \geq \eps ]
\;,
\end{equation}
 where the infimum is taken over all estimators $\hat \theta$ of $\theta$. The above RHS is then
 lower bounded through the following proposition which is a consequence of Fano's lemma,
 see \cite{tsybakov}, Theorem 2.5:

\begin{prop}
\label{Fano}
Assume that $N \geq 2$ and suppose that $\Theta$ contains N+1  elements $\theta_0, ..., \theta_N$ such that:
\begin{enumerate}
\item[(i)] For some $\eps>0$\,, and for any $0 \leq i < j \leq N$\,,
$d(\theta_i, \theta_j) \geq 2\eps$ \; ;
\item[(ii)] For any $j=1,...,N$\,,
$P_j$ is absolutely continuous with respect to $P_0$,  and
\begin{equation}
\label{max}
 \frac{1}{N}\sum_{j=1}^{N} {\cal K}(P_j, P_{0}) \leq \omega  \; \log(N)  \; , 
\end{equation}
for some $0 < \omega < 1/8$.  
\end{enumerate}
Then
\begin{equation*}
\inf_{\hat \theta}\max_{1\leq j \leq N} P_{j}(\; d(\hat \theta, \theta_j) \geq \eps \;) \;  \geq \; \frac{\sqrt{N}}{1+\sqrt{N}}
   \left( 1-2\omega - \sqrt{\frac{2\omega}{\log(N)}} \right) > 0 \; ,
\end{equation*}
where the infimum is taken over all estimators $\hat \theta$ of $\theta$.
\end{prop}

\subsection{Proof of Theorem \ref{minimaxlowerrate}}

We will apply the above general result to our target distance $d_s: \Omega_{\nux}(r,R) \times \Omega_{\nux}(r,R) \longrightarrow \R_{+}$, given by
\[  d_s(f_1, f_2) \; = \; \snorm{f_1 - f_2}_{\h} \; ,  \]
with $s \in [0, \frac{1}{2}]$ and $\nux \in \priorgr(b, \alpha)$\,.
We will establish the lower bounds in the
particular case where the distribution of $Y$ given $X$ is Gaussian with
variance $\sigma^2$ (which satisfies the Bernstein moment condition \eqref{bernstein} with
$M=\sigma$)\,. The main effort is to construct a finite subfamily belonging to the model of interest
and suitably satisfying the assumptions of Proposition~\ref{Fano}; this is the goal of the
forthcoming propositions and lemmata.

\begin{prop}
  \label{prop:lowmodel}
Let $\nux \in \priorgr(b, \alpha)$, for  $b>1$, $\alpha >0$. Assume  that $r>0, R>0$. To each $f \in \Omega_{\nux}(r,R)$ and $x \in \X$ 
we associate the following measure:
\begin{equation}
\label{stochkern}
\rho_f(dx,dy):=\rho_f(dy|x)\nux (dx)\,, \text{ where } \rho_f(dy|x):= {\cal N} (Af(x),\sigma^2)\,.
\end{equation}
Then:
\begin{enumerate}
\item[(i)]
  The measure $\rho_f$ belongs to the class $\M(r,R,\priorgr(b, \alpha))$, defined in $(\ref{measureclass})$.
\item[(ii)]
Given $f_1, f_2 \in \Omega_\nux(r,R)$, the Kullback-Leibler divergence between $\rho_1$ and $\rho_2$ satisfies
\[ \K(\rho_1, \rho_2)  = \frac{1}{2 \sigma^2} \big\| \sqrt{B}(f_1 - f_2) \big\|_{\h}^2 \,. \]
\end{enumerate}
\end{prop}
\begin{proof}
  Point (i) follows directly from the definition of the class $\M(r,R,\priorgr(b, \alpha))$\,. For point (ii), note that
  the Kullback-Leibler divergence between two Gaussian distributions with identical variance $\sigma^2$ and mean difference $\Delta$
  is $\Delta/2\sigma^2$\,. Since $\rho_1,\rho_2$ have the same $X$-marginal $\nux$, it holds
  \begin{multline*}
    {\cal K}(\rho_1,\rho_2) = \E[ {\cal K}(\rho_1(.|X),\rho_2(.|X))] =
    \frac{1}{2\sigma^2} \int \paren{A(f_1 - f_2)(x)}^2 d\nux(x)\\
    = \frac{1}{2 \sigma^2} \norm{ S(f_1-f_2)}^2_{L^2(\nux)} =
    \frac{1}{2 \sigma^2} \big\| \sqrt{B}(f_1 - f_2) \big\|_{\h}^2\,.
  \end{multline*}
\end{proof}

The following lemma is a variant from \cite{optimalratesRLS}, Proposition 6, which will be useful in the 
subsequent proposition. 

\begin{lem}
\label{rademacher}
For any $m \geq 28$ there exist an integer $N_m > 3$ and $\rad_1, ..., \rad_{N_m} \in \{-1,+1\}^m$ such that for any 
$i,j \in \{1,...,N_m\}$ with $i \not = j$ it holds
\begin{equation}
\label{rademacher2}
\log(N_m-1) >  \frac{m}{36}>\frac{2}{3}   \;,
\end{equation}
and
\begin{equation}
\label{rademacher1}
  \sum_{l=1}^m (\rad_i^l - \rad_j^l)^2 \geq m \,,
\end{equation}
where $\rad_i= (\rad_i^1,...,\rad_i^m)$.
\end{lem}

\begin{prop}
 \label{three}
Assume $\nux \in \priorgr(b, \alpha)$\,.
Let $0\leq s \leq 1/2$, $R>0$ and $r>0$.  
For any $\eps_0>0$ there exist $\eps \leq \eps_0$, $N_{\eps} \in \N$ and functions $f_1,..., f_{N_{\eps}} \in \h$ satisfying
\begin{enumerate}
\item[(i)] 
$f_i \in \Omega_\nux({r,R})$ for any $i = 1,...,N_{\eps}$ and   
\[  \norm{\; B^s(f_i - f_j)\;}^2_{\h}> \eps^2 \; , \] 
for any $i,j = 1,...,N_{\eps}$ with $i \not = j$.
\item[(ii)]
Let $\rho_i := \rho_{f_i}$ be given by $(\ref{stochkern})$. Then it holds
\[ \K(\rho_i, \rho_j) \leq C_{b,r,s}\; R^2 \sigma^{-2} \paren{\frac{\eps}{R}}^{\frac{2r+1}{r+s}}
\;,\] 
for any $i,j = 1,...,N_{\eps}$ with $i \not = j$\,.
\item[(iii)] $\log (N_{\eps}-1) \geq C_{\alpha,b,r,s} \paren{\frac{R}{\eps}}^{\frac{1}{b(r+s)}}\,.$
\end{enumerate}
If $\nux$ belongs to the subclass $\priorgr_{strong}(b, \alpha)$, then the assertions from (i), (ii) and (iii) are valid for all 
$\eps >0 $ small enough (depending on the parameters $r, R, s, \alpha, b$ as well as $j_0,\gamma$ coming from the choice
of $\nux$ in $\priorgr_{strong}(b, \alpha)$\,; the multiplicative constants in (ii), (iii) then also depend on $\gamma$\,.)
\end{prop}

\begin{proof}
We first prove the proposition under the stronger assumption that $\nux$ belongs to $\priorgr_{strong}(b,\alpha)$. 
We recall from \eqref{eq:eigdec} that we denote $(e_l)_{l \geq 1}$ an orthonormal family of $\h$ of eigenvectors of $B$
corresponding to the eigenvalues $(\eigv_l)_{l \geq 1}$\,, which satisfy by definition of $\priorgr_{strong}(b,\alpha)$:
\begin{equation}
  \label{eq:eiglb}
\forall l\geq 0 \,:  \eigv_l \geq \alpha l^{-b}\,
\end{equation}
and
\begin{equation}
  \label{eq:eiglbstrong}
\forall l\geq l_0 \,: \eigv_{2l} \geq 2^{-\gamma} \eigv_l   \;,
\end{equation}
for some $l_0 \in \N$ and for some $\gamma>0$. 
For any given $\eps < R 2^{-\gamma (r+s)} \left (\frac{\alpha^{1/b}}{\max(28, l_0)}\right)^{b(r+s)}$ we pick 
$m = m(\eps):= \max \{\;l \geq 1: \; \eigv_l \geq 2^{\gamma} (\eps R^{-1})^{\frac{1}{r+s}} \; \} $. Note that $m \geq \max(28, l_0)$, 
following from the choice of $\eps$ and from \eqref{eq:eiglb}.

Let $N_m > 3$ and $\rad_1,...,\rad_{N_m} \in \{-1,+1\}^m$ be given by Lemma~$\ref{rademacher}$ and define
\begin{equation}
\label{gi}
g_i := \frac{\eps}{\sqrt{m}} \sum_{l=m+1}^{2m} \rad_i^{(l-m)} \left(\frac{1}{\eigv_l}\right)^{r+s}  e_l \; .
\end{equation}
We have by $(\ref{eq:eiglbstrong})$ and from the definition of $m$  
\[
\norm{g_i}^2_{\h}  =  \frac{ \eps^2}{m}\sum_{l=m+1}^{2m} \left(\frac{1}{\eigv_k}\right)^{2(r+s)}
\leq \eps^2 \eigv_{2m}^{-2(r+s)} \leq  \eps^2 2^{2\gamma(r+s)}\eigv_{m}^{-2(r+s)} \leq R^2 \,.\]
For $i = 1,...,N_{m}$ let $f_i:=B^{r}g_i \in \Omega_\nux({r,R})$, with $g_i$ as in $(\ref{gi})$. 
Then  
\begin{align*}
\norm{\; B^s(f_i - f_j)\;}^2_{\h} & =  \norm{\; B^{r+s} (g_i-g_j)\; }^2_{\h} \\
& =  \frac{\eps^2}{m}\sum_{l=m+1}^{2m} (\rad_i^{l-m}-\rad_j^{l-m})^2 \left(\frac{1}{\eigv_l}\right)^{2(r+s)} \eigv_l^{2(r+s)}   
\geq \eps^2\,,
\end{align*}
by $(\ref{rademacher1})$\,, and the proof of $(i)$ is finished. 
For $i = 1,...,N_{\eps}$, let $\rho_i=\rho_{f_i}$ be defined by \eqref{stochkern}. Then, using the definition of $m$, the Kullback-Leibler divergence satisfies 
\begin{align*}
\K(\rho_i, \rho_j) &=  \frac{1}{2\sigma^2} \; \big\| \sqrt{B}(f_i-f_j) \big\|^2_{\h} \\
&= \frac{1}{2 \sigma^2} \; \norm{ B^{r+1/2}(g_i-g_j) }^2_{\h} \\
&= \frac{\eps^2}{2\sigma^2 m}\sum_{l=m+1}^{2m}(\rad_i^{l-m}-\rad_j^{l-m})^2 \left(\frac{1}{\eigv_l}\right)^{2(r+s)} \eigv_l^{2r+1} \\
&\leq 2\sigma^{-2} \eigv_{m+1}^{1-2s} \eps^2 \\
&\leq 2^{1+\gamma(1-2s)}  \; \sigma^{-2} R^{2}  \paren{\frac{\eps}{R}}^{\frac{1+2r}{r+s}} \,,
\end{align*}
which shows $(ii)$. 
Finally, \eqref{rademacher2}, \eqref{eq:eiglb}, \eqref{eq:eiglbstrong} and the definition of $m$ imply
\[ \log(N_m-1) \geq \frac{m}{36} \geq \frac{\alpha^{1/b}}{36} \eigv_m^{-1/b}
\geq \frac{\alpha^{1/b}}{36} 2^{-\gamma/b}\eigv_{2m}^{-1/b} 
\geq \frac{\alpha^{1/b}}{36} 2^{-{2\gamma/b} }\left(\frac{R}{\eps}\right)^{\frac{1}{b(r+s)}} \; ,\]
thus (iii) is established.

We now assume that $\nux$ belongs to $\priorgr(b,\alpha)$ and only satisfies condition $(\ref{eq:eiglb})$.
Let any $\eps_0>0$ be given.
We pick $m \in \N$ satisfying $m\geq 28$ and the two following conditions: 
\begin{equation}
 \label{eq:condbsmall}
   \eigv_m \leq 2^{b+1} (R^{-1} \eps_0)^{\frac{1}{r+s}} \,,
\end{equation}
\begin{equation}
    \label{eq:condbdoubling}
  \frac {\eigv_{2m}}{\eigv_m} \geq 2^{-b-1}\,.
\end{equation}
Since the sequence of eigenvalues $(\eigv_m)$ converges to $0$, condition \eqref{eq:condbsmall} must be satisfied for any $m$ big enough, 
say $m \geq m_0(\eps_0)$. Subject to that condition, we argue by contradiction
that there must exist $m$ satisfying \eqref{eq:condbdoubling}.
If that were not the case, we would have by immediate recursion
for any $l>0$, introducing $m':=2^l m_0(\eps_0)$:
\[
\eigv_{m'} < 2^{-l(b+1)} \eigv_{m_0(\eps_0)} = \paren{\frac{m'}{m_0(\eps_0)}}^{-b-1} \eigv_{m_0(\eps_0)}
= C_{\eps_0} (m')^{-(b+1)}\,,
\]
which would (eventually, for $l$ big enough) contradict \eqref{eq:eiglb}\,.
Therefore, there must exist an $m>m_0$ satisfying the required conditions. Now put
\begin{equation}
  \label{eq:defepsp}
\eps := 2^{-(b+1)(r+s)} R \eigv_m^{r+s} \leq \eps_0\,,
\end{equation}
where the inequality is from requirement \eqref{eq:condbsmall}. 
For $i = 1,...,N_m$, we define $g_i$ as in \eqref{gi}. Then 
$||g_i||_{\h} \leq R$. Again, let $f_i:=B^{r}g_i \in \Omega_\nux({r,R})$ and the same calculations as above 
(with $\gamma$ replaced by $b+1$) lead to $(i), (ii)$ and $(iii)$.
\end{proof}

Now we are in the position to prove the minimax lower rate.

\begin{proof}[Proof of Theorem \ref{minimaxlowerrate}] 
Let the parameters $r,R,s,b,\alpha,\sigma$ be fixed for the rest of the proof,
and the marginal distribution $\nux \in \priorgr(b,\alpha)$ also be fixed.

 Our aim is to apply Proposition~\ref{Fano} to the distance $d_s(f_1, f_2) := \snorm{f_1 - f_2}_{\h}$ ($s \in [0, \frac{1}{2}]$\,), on the class $\Theta:=\Omega_{\nux}(r,R)$\,,
where for any $f\in \Omega_{\nux}(r,R)$\,, the associated distribution is $P_f := \rho_f^{\otimes n}$
with $\rho_f$ defined as from Proposition~\ref{prop:lowmodel} (i)\,;
more precisely, we will apply this proposition along a well-chosen sequence $(n_k,\eps_k)_{k\geq 0}$\,.
From Proposition~\ref{three}\,, we deduce that
 there exists a decreasing
null sequence $(\eps_k) >0$
such that for
any $\eps$ belonging to the sequence, there exists $N_{\eps}$ and functions $f_1 ,..., f_{N_{\eps}}$
satisfying (i)-(ii)-(iii).
In the rest of this proof, we assume $\eps=\eps_k$ is a value belonging to the null sequence.
Point (i) gives requirement (i) of Proposition~\ref{Fano}. 
We turn to requirement \eqref{max}. Let $\rho_j = \rho_{f_j}$ be given by $(\ref{stochkern})$. Then 
by Proposition \ref{three} (ii)-(iii)\,:
\begin{align*}
\frac{1}{N_{\eps}-1} \sum_{j=1}^{N_{\eps}-1}{\cal K}(\rho^{\otimes n}_j, \rho^{\otimes n}_{N_{\eps}}) &=
\frac{n}{N_{\eps}-1} \sum_{j=1}^{N_{\eps}-1}{\cal K}(\rho_j, \rho_{N_{\eps}})  \\
& \leq n C_{b,r,s}\; R^2 \sigma^{-2}\paren{\frac{\eps}{R}}^{\frac{2r+1}{r+s}}\\
& \leq n C_{\alpha,b,r,s}\; R^2 \sigma^{-2}
\paren{\frac{\eps}{R}}^{\frac{2br+b+1}{b(r+s)}} \log(N_{\eps }-1) \\
& =:  \omega \;\log(N_{\eps }-1)\,.
\end{align*}
Choosing $n:= \left\lfloor \paren{8C_{\alpha,b,r,s}\; R^2 \sigma^{-2}
  \paren{\eps R^{-1}}^{\frac{2br+2r+1}{b(r+s)}}}^{-1} \right\rfloor$ ensures $\omega \leq \frac{1}{8}$
and therefore requirement \eqref{max} is satisfied.
Then Proposition~\ref{Fano} entails:
\begin{eqnarray*}
\inf_{\hat f_{\bullet}} \max_{1\leq j\leq N_{\eps}} \rho_j^{\otimes n} \left( \; \big\| B^s (\hat f_{\bullet} - f_j) \big\|_{\h}  \geq \frac{\eps}{ 2}\; \right) &\geq  &
       \frac{\sqrt{N_{\eps}-1}}{1 + \sqrt{N_{\eps}-1}} \left( 1-2\omega- \sqrt{\frac{2\omega}{\log{ (N_{\eps}-1})}} \right) \\
&\geq & \frac{1}{2}\left(\frac{3}{4}- \sqrt{\frac{3}{8}}\right) \\
& > & 0 \;.
\end{eqnarray*}
This inequality holds for any $(n_k,\eps_k)$ for $\eps_k$ in the decreasing null sequence and
$n_k$ given by the above formula; we deduce that $n_k \rightarrow \infty$
with
\[
\eps_k \geq C_{\alpha,b,r,s} R \paren{\frac{\sigma}{R\sqrt{n_k}}}^{\frac{2b(r+s)}{2br+b+1}}\,.
\]
Thus, applying $(\ref{reduction})$ and taking the limsup gives the result.

Now suppose that that $\nux$ belongs to $\priorgr_{strong}(b,\alpha)$. 
Define  $\eps := R(8C_{\alpha, b, r, s})^{-\frac{b(r+s)}{2br+b+1}}\left(\frac{\sigma^2}{R^2 n} \right)^{\frac{b(r+s)}{2br+b+1}}$\,,
then for any $n$ sufficiently large, points (i)-(ii)-(iii) of Proposition~\ref{three} will hold.
The same calculations as above now hold for any $n$ large enough; finally taking the liminf finishes the proof. 
\end{proof}

\begin{proof}[Proof of Corollary~\ref{maincor2}]
The main point is only to ensure that the {\em strong} minimax lower bound applies,
for this we simply check that ${\priorgr_{strong}}(b, \alpha)  \supset \PPP'=\priorle(b, \beta) \cap {\priorgr}(b, \alpha)$\,.
For any $\nux\in  \priorle(b, \beta) \cap {\priorgr}(b, \alpha)$\,, the eigenvalues of the operator $B_\nux$
satisfy $\alpha j^{-b} \leq \eigv_j \leq \beta j^{-b}$ for all $j\geq 1$\,. It follows that
for any $j \geq 1$\,:
\[
\frac{\eigv_{2j}}{\eigv_{j}} \geq \frac{\alpha}{\beta} 2^{-b}\,,
\]
so that the conditions for $\nux \in \priorgr_{strong}(b, \alpha)$ are met (with parameters $\gamma:= b + \log_2 \frac{\beta}{\alpha}$\,,
$l_0=1$). Since $\PPP'$ is assumed to be nonempty, for any $\nux \in \PPP'$ the strong lower minimax bound
of Theorem~\ref{minimaxlowerrate} applies to the family $\M_{R,M,\sigma} := \M(r,R,\set{\nu})$
and a fortiori to the family $\M_{R,M,\sigma} := \M(r,R,\PPP')$ whose models are larger.
On the other hand since $\M(r,R,\PPP') \subset {\M}(r,R, {\priorle}(b, \beta))$ the upper bound of Theorem~\ref{maintheo3}
applies and we are done.
\end{proof}


\appendix

\section{Proof of Concentration Inequalities}
\label{app:concentration}

\begin{prop}
\label{concentration2}
Let $(Z , {\cal B}, \PP )$ be a probability space and $\xi$ a random variable on $Z$ with values in a real 
separable Hilbert space ${\cal H}$. Assume that there are two positive constants $L$ and $\sigma$ such that for any $m\geq 2$
\begin{equation}
\label{expecm}
\E\big[ \norm{ \xi - \E[\xi]    }_{\h}^m  \big ] \leq  \frac{1}{2}m!\sigma^2L^{m-2}.
\end{equation}
If the sample $z_1,...,z_n$ is drawn i.i.d. from $Z$ according to $\PP$, then, for any $0<\eta<1$, with probability greater than $1-\eta$
\begin{equation}
\label{upbound2}
\Big\|   \frac{1}{n}\sum_{j=1}^n\xi(z_j)-\E[\xi]   \Big\|_{\cal H} \leq \delta (n,\eta),
\end{equation}
where 
\begin{eqnarray*}
\delta (n, \eta ) & := & 2 \log(2\eta^{-1}) \left( \frac{L}{n} + \frac{\sigma}{\sqrt n} \right). 
\end{eqnarray*}
In particular, $(\ref{expecm})$ holds if
\begin{eqnarray*}
\norm{ \xi (z) }_{\h}&\leq & \frac{L}{2} \qquad a.s. \; ,\\
\E\big[ \norm{\xi }^2_{\h} \big]&\leq & \sigma^2.
\end{eqnarray*}
\end{prop}

\begin{proof}
See \cite{optimalrates}, from the original result of \cite{pinelissakha} (Corollary 1)\,.
\end{proof}


\begin{proof}[Proof of Proposition \ref{Geta1}]
Define $\xi_1: \X \times \R \longrightarrow \h$ by 
\begin{equation*}
 \xi_1 (x,y) := (\bar B + \lam)^{-1/2}\kappa^{-2} (y-Af(x)) F_x \; ,
\end{equation*}
abusing notation we also denote $\xi_1$ the random
variable $\xi_1(X,Y)$ where $(X,Y) \sim \rho$\,.
The model assumption~\eqref{basicmodeleq} implies
\begin{align*}
 \E[ \xi_1   ]  &= \kappa^{-2} (\bar B + \lam)^{-1/2}\; \int_{\X}F_x\int_{\R}(y-A\fo(x))\; \rho(dy|x)\nux(dx) \\
&= (B + \lam)^{-1/2}\; \int_{\X}F_x(A\fo(x)-A\fo(x)) \; \nux(dx)\\
&= 0 \,,
\end{align*}
and therefore 
\begin{align*}
\frac{1}{n}\sum_{j=1}^n\xi_1 (x_j, y_j)- \E[\xi_1 ]& =  \frac{1}{n}\sum_{j=1}^n 
(\bar B + \lam)^{-1/2} \kappa^{-2} ( y_{j}- A\fo(x_j))F_{x_j} \\
& =  (\bar B + \lam)^{-1/2} \kappa^{-2} S_{\x}^{\star} \;\left( \y  - S_{\x}\fo \right) \;.\\       
& =  (\bar B + \lam)^{-1/2}\;\left(  \bar S_{\x}^{\star}\y  - \bar B_{\x}\fo \right)  \;.       
\end{align*}
Moreover, by assumption $(\ref{bernstein})$\,, for $m\geq 2$:
\begin{align*}
  \E[\;\norm{\xi_1}^m_{\h}\;] &
  = \int_{\X\times\R} \norm{\kappa^{-2} (y-A\fo(x)) (\bar B + \lam)^{-1/2} F_x}^m_{\h}
  \rho(dx,dy)\\
  & = \int_{\X} \kappa^{-2m} \norm{(\bar B + \lam)^{-1/2} F_x}_{\h}^m \int_{\R} \abs{y-A\fo(x)}^m \;\rho(dy|x) \nux(dx)  \\
  & \leq  \frac{1}{2}m! \sigma^2 M^{m-2} \kappa^{-2m} \sup_{x \in \X}\norm{(\bar B+\lam)^{-1/2}F_x}^{m-2}_{\h}
  \int_{\X} \tr\paren{ (\bar B + \lam)^{-1}F_x \otimes F_x^{\star} } \; \nux(dx)  \\
  &\leq  \frac{1}{2}m! \kappa^{-m} \sigma^2 M^{m-2}  \lambda^{-\frac{m-2}{2}}
  \tr \paren{ (\bar B+\lam)^{-1/2} \kappa^{-2} \int_\X F_x \otimes F_x^{\star}\; \nux(dx) } \\
  & = \frac{1}{2}m! \left( \kappa^{-1} \sigma \sqrt{{\cal N}(\lam)} \right)^2
  \paren{\frac{\kappa^{-1}M}{\sqrt{\lambda}}}^{m-2}\,.
\end{align*}
As a result, Proposition $\ref{concentration2}$ implies with probability at least $1-\eta$
\begin{align*}
  \norm{ (\bar B + \lam)^{-1/2} ( \bar B_{\x}\fo-\bar S_{\x}^{\star}\y) }_{\h}  \leq
  \delta_1\left( n,\eta \right)  \; , 
\end{align*}
where  
\[
\delta_1(n,\eta) = 2\log(2\eta^{-1}) \kappa^{-1} \left( \frac{M}{n\sqrt{\lam}} + \frac{\sigma}{\sqrt{ n}}\sqrt{{\cal N}(\lam)}\right) .
\]
For the second part of the proposition, we introduce similarly
$\xi'_1(x,y):=\kappa^{-2} (y-Af(x)) F_x$\,, which satisfies
\[
\E[ \xi_1']=0 \,; \qquad \frac{1}{n}\sum_{j=1}^n\xi'_1 (x_j, y_j)- \E[\xi'_1 ]=
 \bar S_{\x}^{\star}\y  - \bar B_{\x}\fo\,,
 \]
 and
 \[
 \E\big[\norm{\xi'_1}^m_{\h}\big] \leq \kappa^{-m} \E\big[ \abs{y-Af(x)}^m\big] \leq
 \frac{1}{2}m! \left( \kappa^{-1} \sigma\right)^2
  \paren{\kappa^{-1}M}^{m-2}\,.
  \]
Applying  Proposition $\ref{concentration2}$ yields the result.

\end{proof}


\begin{proof}[Proof of Proposition \ref{Geta2}]
We proceed as above by defining $\xi_2: \X  \longrightarrow \hs(\h)$ (the latter
denoting the space of Hilbert-Schmidt operators on $\h$) by 
\[\xi_2(x):= (\bar B+\lam)^{-1} \kappa^{-2} B_x\,, \]
where $B_x := F_x \otimes F_x^\star$\,. 
We also use the same notation $\xi_2$ for the random variable $\xi_2(X)$ with $X\sim \nux$\,. Then, 
\[
\E[\xi_2] = (\bar B+\lam)^{-1} \kappa^{-2} \int_{\X} B_x \; \nux(dx) 
          = (\bar B+\lam)^{-1}\bar B \,,
\]
and therefore
\[
\frac{1}{n}\sum_{j=1}^n\xi_2(x_j) - \E[\xi_2]= (\bar B+\lam)^{-1}(\bar B- \bar B_{\x})\; .
\]
Furthermore, since $\bar B_x$ is of trace class and  $(\bar B+\lam)^{-1}$ is bounded, we have using Assumption \ref{assumpteval}
\[
\norm{ \xi_2 (x)}_{\hs} \leq  \norm{(\bar B+\lam)^{-1}} \kappa^{-2} \norm{B_x}_{\hs} 
\leq  \lam^{-1} =: L_2/2  \;,
\]
uniformly for any $x \in \X$. Moreover,
\begin{align*}
\E\big[ \norm{\xi_2}^2_{\hs} \big] & =  \kappa^{-4} \int_{\X} \tr\left( \; B_x(\bar B+\lam )^{-2} B_x\; \right) \nux(dx) \\
&\leq   \norm{(\bar B+\lam)^{-1}} \kappa^{-4} \int_{\X}\norm{B_x}\; \tr\left( \;(\bar B+\lam)^{-1}  B_x \; \right) \nux(dx) \\
& \leq   \frac{{\cal N}(\lam)}{\lam} = :\sigma_2^2 \; . 
\end{align*}
Thus, Proposition $\ref{concentration2}$ applies and gives with probability at least $1-\eta$
\[
\norm{ (\bar B+\lam)^{-1}(\bar B- \bar B_{\x})  }_{\hs} \; \leq \; \delta_2 (n,\eta) 
\]
with
\[
\delta_2(n,\eta ) = 2\log(2\eta^{-1}) \left( \frac{2}{n \lam} + \sqrt{\frac{{\cal N}(\lam)}{n\lam}} \right)\; .
\]
\end{proof}


\begin{proof}[Proof of Proposition \ref{neumann}]
We write the Neumann series identity
\[
  (\bar B_{\x} + \lam)^{-1}(\bar B + \lam ) = (I - \bar S_{\x}(\lam))^{-1} 
  = \sum_{j=0}^{\infty} \bar S_{\x}^j(\lam) \; ,
\]
with 
\[ \bar S_{\x}(\lam) = (\bar B+ \lam)^{-1}(\bar B - \bar B_{\x}) \; ; \]
it is well known that the series converges in norm 
provided that $\norm{\bar S_{\x}(\lam) } < 1 $. In fact, applying Proposition~\ref{Geta2} 
gives with probability at least $1-\eta$\,:
\[
\norm{ \bar S_{\x}(\lam) } \; \leq \; 2\log(2\eta^{-1}) \left( \frac{2}{n \lam} + \sqrt{\frac{{\cal N}(\lam)}{n\lam}}  \right) \; .
\]
Put $C_\eta:= 2\log(2\eta^{-1}) >1$ for any $\eta \in (0,1)$\,. Assumption~\eqref{assumpt2b} reads
$\sqrt{n\lam} \geq 4 C_\eta \sqrt{\max({\cal N}(\lambda),1)}$\,, implying
$\sqrt{n\lam} \geq 4 C_\eta \geq 4$ and therefore $\frac{2}{n\lambda} \leq \frac{1}{2\sqrt{n\lambda}} \leq \frac{1}{8C_\eta}$\,, hence
\[
C_\eta \left( \frac{2}{n \lam} + \sqrt{\frac{{\cal N}(\lam)}{n\lam}}  \right) 
 \leq C_\eta \paren{ \frac{1}{8C_\eta} + \frac{1}{4C_\eta}}  <  \frac{1}{2} \; .
\]
Thus, with probability at least $1-\eta$:
\begin{equation*}
\norm{(\bar B_{\x} + \lam)^{-1}(\bar B + \lam ) }  \leq  2 \; .
\end{equation*}
\end{proof}


\begin{proof}[Proof of Proposition \ref{rem1}]
Defining $\xi_3: \X  \longrightarrow \hs(\h)$ by 
\[\xi_3(x):= \kappa^{-2} F_x \otimes F_x^\star = \kappa^{-2} B_x \]
and denoting also, as before, $\xi_3$ for the random variable $\xi_3(X)$ (with $X\sim \nux$)\,,
we have $\E[\xi_3] = \bar B $ and therefore
\begin{equation*}
\frac{1}{n}\sum_{j=1}^n\xi_3(x_j) - \E[\xi_3]= (\bar B_{\x}- \bar B)\; .
\end{equation*}
Furthermore, by Assumption $\ref{assumpteval}$
\begin{equation*}
\norm{ \xi_3 (x)}_{\hs}  =  \kappa^{-2} \norm{F_x}^2 \leq  1 =: \frac{L_3}{2} \qquad {\rm a.s.} \;,
\end{equation*}
also leading to $\E[ \norm{\xi_2}^2_{\hs} ] \leq 1 =: \sigma_3^2$\,.
Thus, Proposition $\ref{concentration2}$ applies and gives with probability at least $1-\eta$
\begin{equation*}
\norm{ \bar B- \bar B_{\x}  }_{\hs} \; \leq \; 6 \log(2\eta^{-1})\; \frac{1}{\sqrt n} \;.
\end{equation*}
\end{proof}


\section{Perturbation Result}
\label{app:perturbation}

The estimate of the following proposition is crucial for proving the upper bound in case the source condition is of H\"older type $r$ with $r>1$.
We remark that for $r>1$ the function $t \mapsto t^r$ is not operator monotone. One might naively expect estimate $(\ref{est:schatten})$ to hold for a constant $C$ 
given by the Lipschitz-constant of the scalar function $t^r$. As shown in \cite{Bat2000}, this is false even for finite dimensional positive matrices. The point of Proposition 
\ref{app:pert:prop} is that $(\ref{est:schatten})$ still holds for some larger constant depending on $r$ and the upper bound of the spectrum. We do not expect this
result to be particularly novel, but tracking down a proof in the literature proved
elusive, not to mention that occasionally erroneous statements about related issues can be 
found. For this reason we here provide a self-contained proof for completeness sake.

\begin{prop}
\label{app:pert:prop}
Let $B_1, B_2$ be two non-negative self-adjoint operators on some Hilbert space with $||B_j|| \leq a$, $j=1,2$, for some $a>0$. Assume the
$B_j$ belong to the Schatten class $\es^p$ for $1 \leq p \leq \infty$. If $1 < r$, then
\begin{equation}
\label{est:schatten}
||B_1^r-B_2^r||_p \; \leq \; C_{r}a^{r-1} \; || B_1 - B_2 ||_p \; \; ,
\end{equation}
for some $C_{r}<\infty$. This inequality also holds in operator norm for non-compact bounded (non-negative and self-adjoint)  $B_j$.
\end{prop}

\begin{proof}
We extend the proof of \cite{fuku}, given there in the case $r=3/2$ in operator norm. We also restrict ourselves to the case $a=1$. 
On ${\cal D}:=\{z: \; |z| \leq 1 \}$, we consider the functions $f(z)=(1-z)^{r}$ and $g(z)=(1-z)^{r-1}$.  
The proof is based on the power series expansions 
\[  f(z)=\sum_{n\geq 0}^{\infty}b_nz^n \quad \mbox{and} \quad g(z)= \sum_{n \geq 0}^{\infty}c_nz^n  \; ,\]
which converge absolutely on ${\cal D}$. To ensure absolute convergence on the boundary $|z|=1$, notice that 
\[ c_n=\frac{1}{n!}g^{(n)}(0) =  \frac{(-1)^n}{n!}\prod_{j=1}^{n}(r-j)\,,
\]
so that all coefficients $c_n$ for $n\geq r$ have the same sign $s := (-1)^{\lfloor r \rfloor}$
(if $r$ is an integer these coefficients vanish without altering the argument below)
implying for any $N>r$\,: 
\begin{align*}
\sum_{n=0}^N |c_n|  = & \sum_{n=0}^{\lfloor r \rfloor}|c_n|  + s \sum_{n=\lfloor r \rfloor +1}^Nc_n  = \sum_{n=0}^{\lfloor r \rfloor}|c_n| + s \lim_{z \nearrow  1}\sum_{n=\lfloor r \rfloor +1}^Nc_nz^n \\
              \leq & \sum_{n=0}^{\lfloor r \rfloor}|c_n| + s\lim_{z \nearrow  1}\Big( g(z)- \sum_{n=0}^{\lfloor r \rfloor} c_n \Big) \\
            = & 2\sum_{i=0}^{ \lfloor r/2 \rfloor}|c_{\lfloor r \rfloor - 2i} | \;.
\end{align*}
A bound for $\sum_n |b_n|$ can be derived analogously. Since $f(1-B_j) = B_j^r$, we obtain
\[ \| B_1^r- B_2^r \|_p \leq \sum_{n=0}^{\infty} |b_n|\; \| (I-B_1)^n - (I-B_2)^n  \|_p \; .\]
Using the algebraic identity $ T^{n+1} - S^{n+1} = T(T^n-S^n) + (T-S)S^n $\,, the triangle inequality and
making use of $\|TS\|_p \leq \|T\|\; \|S\|_p$  for $S \in \es^p$, $T$ bounded,  
the reader can easily convince himself by induction that
\begin{itemize}
\item for $j=1,2$, $B_j \in \es^p$ imply $(I-B_1)^n - (I-B_2)^n \in \es^p$ and
\item $\| (I-B_1)^n - (I-B_2)^n  \|_p \leq n\|B_1 - B_2 \|_p$.
\end{itemize}
From $f'(z)=-r g(z)$ we have the relation $|b_n|= \frac{r}{n}|c_{n-1}|$, $n \geq 1$. Collecting all pieces leads to
\[ \| B_1^r- B_2^r \|_p \leq  \| B_1 - B_2\|_p \sum_{n=0}^{\infty}n|b_n| = r  \| B_1 - B_2\|_p  \sum_{n=0}^{\infty}|c_n| \;.\]
\end{proof}


\section{Auxiliary technical lemmata}

\begin{lem}
  \label{le:boring}
  Let $X$ be a nonnegative real random variable such that the following holds:
  \begin{equation}
     \label{eq:F}
     \PP[X > F(t)] \leq t \,, \text{ for all } t \in (0,1]\,,
  \end{equation}
  where $F$ is a monotone nonincreasing function
  $(0,1] \rightarrow \R_+$\,. Then
   \[
          \E[X] \leq \int_{0}^1 F(u) du\,.
   \]
\end{lem}

\begin{proof}
  An intuitive, non-rigorous proof is as follows. Let $G$ be the tail distribution function
  of $X$\,, then it
  is well known that $\E[X] = \int_{\R_+} G$\,. Now it seems clear that $\int_{\R_+} G
  = \int_0^1 G^{-1} $\,, where $G^{-1}$ is the upper quantile function for $X$\,.
  Finally $F$ is an upper bound on $G^{-1}$\,.
  
  Now for a rigorous proof, we can assume without loss of generality that $F$ is left continuous:
  replacing $F$ by its left limit in all points of $(0,1]$ can only make it larger since it is
    nonincreasing, hence \eqref{eq:F} is still satisfied, moreover since a monotone function
    has an at most countable number of discontinuity points, this operation does not change the
    value of the integral $\int_0^1 F$\,.  Define the following pseudo-inverse for $x\in \R_+$\,:
  \[
  F^{\dagger}(x) := \inf \set{ t\in(0,1] : F(t) < x}\,,
  \]
  with the convention $\inf \emptyset = 1$\,. Denote $\wt{U} := F^{\dagger}(X)$\,.
  From the definition of $F^\dagger$ and the monotonicity of $F$ it holds that $F^{\dagger}(x) < t
  \Rightarrow x>F(t)$ for all $(x,t) \in \R_+ \times (0,1]$\,. Hence for any $t\in (0,1]$
      \[
      \PP[ \wt{U} < t] \leq \PP[ X > F(t) ] \leq t\,,
      \]
  implying that for all $t \in [0,1]$\,, $\PP[ \wt{U} \leq  t] \leq t$ \,, i.e. $\wt{U}$ is
  stochastically larger than a uniform variable on $[0,1]$.
  Furthermore, by left continuity of $F$\,,
  one can readily check that $F(F^\dagger(x)) \geq x$ if $x \leq F(0)$\,.
  Since $\PP[X > F(0)] = 0$\,, we can replace $X$ by $\wt{X} := \min(X,F(0))$ without
  changing its distribution (nor that of $\wt{U}$). With this modification, in then holds that
  $ F(\wt{U}) = F(F^\dagger(\wt{X})) \geq \wt{X}$\,. Hence
  \[
  \E[X] = \E[\wt{X}] \leq \E[F(\wt{U})] \leq \E[F(U)] = \int_0^1 F(u) du\,,
  \]
  where $U$ is a uniform variable on $[0,1]$, and the second equality holds since $F$ is nonincreasing.
\end{proof}

\begin{cor}
  \label{cor:boring}
  Let $X$ be a nonnegative random variable and $t_0 \in (0,1)$ such that the following holds:
  \begin{equation}
    \label{eq:F1}
    \PP[X > a + b \log t^{-1}] \leq t \,, \text{ for all } t \in (t_0,1]\,, \text{ and }
  \end{equation}
  \begin{equation}
    \label{eq:F2}
    \PP[X > a' + b' \log t^{-1}] \leq t \,, \text{ for all } t \in (0,1]\,,
  \end{equation}
  where $a,b,a',b'$ are nonnegative numbers.
  Then for any $p\leq \frac{1}{2} \log t_0^{-1}$\,:
  \[
  \E[X^p] \leq C_p \paren{ a^p + b^p \Gamma(p+1) + t_0 \paren{ (a')^p + 2(b' \log t_0^{-1})^p }}\,,
    \]
  with $C_p:= \max(2^{p-1},1)$\,.
\end{cor}

\begin{proof}
  Let $F(t) := \ind{t \in (t_0,1]} (a + b \log t^{-1}) + \ind{t \in (0,t_0]} (a' + b' \log t^{-1})$\,.
  Then $F$ is nonnegative, nonincreasing on $(0,1]$ and
    \[ \PP[ X^p > F^p(t)] \leq t\] for all $t \in (0,1]$\,. Applying Lemma~\ref{le:boring}\,, we find
      \begin{equation}
        \label{eq:boring}
   \E[X^p] \leq \int_0^{t_0} (a' + b' \log t^{-1})^p dt + \int_{t_0}^1 (a + b \log t^{-1})^p dt\,.   
   \end{equation}
  Using $(x+y)^{p-1} \leq C_p (x^{p-1}+y^{p-1})$ for $x,y \geq 0$\,, where $C_p = \max(2^{p-1},1)$\,,  
  we upper bound the second integral in \eqref{eq:boring} via
  \[
  \int_{t_0}^1 (a + b \log t^{-1})^p dt
  \leq C_p \paren{ a^p + b^p \int_0^1 (\log t^{-1})^p dt} = C_p \paren{ a^p + b^p \Gamma(p+1)} \,.
  \]
  Concerning the first integral in \eqref{eq:boring}, we write similarly
    \begin{multline*}
  \int^{t_0}_0 (a' + b' \log t^{-1})^p dt
  \leq C_p \paren{ t_0 (a')^p + (b')^p \int_0^{t_0} (\log t^{-1})^p dt}\\
  =  C_p \paren{ t_0 (a')^p + (b')^p \Gamma(p+1,\log t_0^{-1})} \,,
  \end{multline*}
    by the change of variable
  $u = \log t^{-1}$\,, where
    $\Gamma$ is the incomplete
  gamma function. We use the following coarse bound: it can be checked
  that $t\mapsto t^pe^{-\frac{t}{2}}$ is decreasing for $t \geq 2p$\,,
  hence, putting $x:= \log t_0^{-1}$\,,
  \[
  \Gamma(p,x) = \int_{x}^\infty t^p e^{-t}dt \leq x^p e^{-\frac{x}{2}} \int_x^\infty
  e^{-\frac{t}{2}} dt = 2 x^p e^{-x} = 2 t_0 (\log t_0^{-1})^{p} \,,
  \]
  provided $x = \log t_0^{-1} \geq 2p$\,. Collecting all the above pieces we get the conclusion.
\end{proof}



\bibliography{bibliography}
\bibliographystyle{abbrv}


\checknbdrafts

\end{document}